\documentclass[10pt,journal,compsoc]{IEEEtran}

\ifCLASSOPTIONcompsoc
\usepackage[nocompress]{cite}
\else
\usepackage{cite}
\fi
\usepackage{graphicx}
\usepackage{comment}
\usepackage{amsmath,amssymb}
\usepackage{color}
\usepackage{epsfig}
\usepackage{booktabs}
\usepackage{multirow}
\usepackage{array}
\usepackage{subfigure}
\usepackage{setspace}
\usepackage{wrapfig}
\usepackage{arydshln}
\usepackage{algorithm} 
\usepackage{algorithmic}
\usepackage{autobreak}
\usepackage{ragged2e}
\usepackage{scalerel}
\usepackage{enumitem}
\usepackage{amsthm}
 
\newtheorem{theorem}{Theorem}

\newtheorem{lemma}{Lemma}

\usepackage[justification=justified]{caption}
\usepackage[colorlinks,linkcolor=red]{hyperref}
\usepackage[switch]{lineno}

\begin{document}

	\title{Adversarial Reciprocal Points Learning for \\ Open Set Recognition}
	
	\author{Guangyao~Chen,
		Peixi~Peng, Xiangqian~Wang
		and~Yonghong~Tian,~\IEEEmembership{Senior~Member,~IEEE}
		\IEEEcompsocitemizethanks{\IEEEcompsocthanksitem G. Chen, P. Peng,and Y. Tian are
			with Department of Computer Science and Technology, Peking University, Beijing, China. P. Peng and Y. Tian are also with the Peng Cheng Laborotory, Shenzhen, China. \protect \\
			E-mail: \{gy.chen, pxpeng, yhtian\}@pku.edu.cn.
			\IEEEcompsocthanksitem X. Wang is with AI Application Research Center, Huawei, Shenzhen, China. E-mail: basileus.wang@huawei.com.}
	}
	
	\markboth{}%
	{Shell \MakeLowercase{\textit{et al.}}: Bare Demo of IEEEtran for Computer Society Journals}
	
	\IEEEtitleabstractindextext{%
		\begin{abstract}
			\justifying
			Open set recognition (OSR), aiming to simultaneously classify the seen classes and identify the unseen classes as 'unknown', is essential for reliable machine learning. 
			The key challenge of OSR is how to reduce the empirical classification risk on the labeled known data and the open space risk on the potential unknown data simultaneously. 
			To handle the challenge, we formulate the open space risk problem from the perspective of multi-class integration, and model the unexploited extra-class space with a novel concept \emph{Reciprocal Point}.
			Follow this, a novel learning framework, termed \textbf{A}dversarial \textbf{R}eciprocal \textbf{P}oint \textbf{L}earning (ARPL), is proposed to minimize the overlap of known distribution and unknown distributions without loss of known classification accuracy. 
			Specifically, each reciprocal point is learned by the extra-class space with the corresponding known category, and the confrontation among multiple known categories are employed to reduce the empirical classification risk.
			Then, an adversarial margin constraint is proposed to reduce the open space risk by limiting the latent open space constructed by reciprocal points.
			To further estimate the unknown distribution from open space, an instantiated adversarial enhancement method is designed to generate diverse and confusing training samples, based on the adversarial mechanism between the reciprocal points and known classes. This can effectively enhance the model distinguishability to the unknown classes. 
			Extensive experimental results on various benchmark datasets indicate that the proposed method is significantly superior to other existing approaches and achieves state-of-the-art performance. 
			The code is released on \href{https://github.com/iCGY96/ARPL}{github.com/iCGY96/ARPL}.
		\end{abstract}
		
		\begin{IEEEkeywords}
			Open Set Recognition, Out-of-Distribution Detection, Reciprocal Points, Generative Adversarial Learning
	\end{IEEEkeywords}}

	\maketitle
	
	\IEEEdisplaynontitleabstractindextext
	
	\IEEEpeerreviewmaketitle
	
	\IEEEraisesectionheading{\section{Introduction}\label{sec:introduction}}
	
	\IEEEPARstart{O}{ver} the past few years, deep learning has equaled and even surpassed human-level performances in many image recognition and classification tasks \cite{he2015delving}.
	These methods follow the closed set setting which assumes that all testing classes are known or seen in the training.
	However, in realistic applications, the knowledge of the classes is incomplete, and unknown classes may be submitted to an algorithm during testing.
	For example, an autonomous mobile agent such as a self-driving vehicle will probably encounter objects of unknown origin at some point during its lifecycle.
	Therefore, these superhuman performances with closed set settings are illusory, since open set recognition (OSR) \cite{scheirer2012toward} is the setting humans operate in, where they vastly outperform all current computer vision approaches.
	Hence, a robust recognition system should identify test samples as known or unknown, and correctly classify all test instances of seen or known classes simultaneously.
	
	The key to OSR is to reduce the empirical classification risk on labeled known data and the open space risk on potential unknown data simultaneously, where the open space risk is that of labeling the open space as "positive" for any known class \cite{scheirer2012toward}.
	Following this, a typical deep-learning-based baseline employs a linear classification layer and the softmax function on the embedding features to produce a probability distribution over the known classes.
	It typically assumes that the samples from the unknown classes should have a uniform probability distribution over the known classes. 
	As shown in Fig.~\ref{fig:intro_softmax}, softmax constructs several hyperplanes to separate the embedding feature space into different subspaces where each subspace corresponds to a known class.
	For the OSR task, the learned embedding features should be not only separable but also discriminative and sufficiently generalized to identify new unseen classes without label prediction. 
	However, softmax loss only encourages the separability of features, and cannot distinguish the known and unknown classes sufficiently well. 
	To make the features more discriminative, several methods \cite{yang2018robust, wen2016discriminative, Prototype} utilize a prototype to represent each known class in the embedding feature space and encourage the features of the training samples to be close to the corresponding prototypes. 
	As shown in Fig.~\ref{fig:intro_GCPL}, the learned prototypes may converge in the space of the unknown classes in the training, making the known and unknown classes indistinguishable. Overall, these two types of methods both only focus on the known data and ignore the potential characteristics of the unknown data, resulting in less effectiveness in reducing the open space risk.
	Contrary to the infinite unknown image space in Fig.~\ref{fig:image}, most unknown samples obtain lower deep magnitude features from the neural network \cite{dhamija2018reducing}, because these samples could not activate a model trained with finite known samples.
	This can also be observed from Fig.~\ref{fig:intro}. 
	Hence, we argue that not only the known classes but also the potential unknown deep space should be modeled in the training.
	
	\begin{figure*}[!tb]
		\centering
		\subfigure[Image-space]{
			\begin{minipage}[b]{0.28\linewidth}
				\centering
				\includegraphics[width=\linewidth]{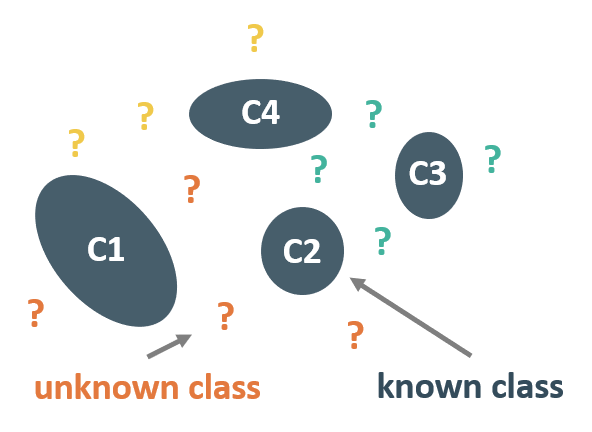}
				\label{fig:image}
			\end{minipage}%
		}%
		\subfigure[Softmax]{
			\begin{minipage}[b]{0.23\linewidth}
				\centering
				\includegraphics[width=\linewidth]{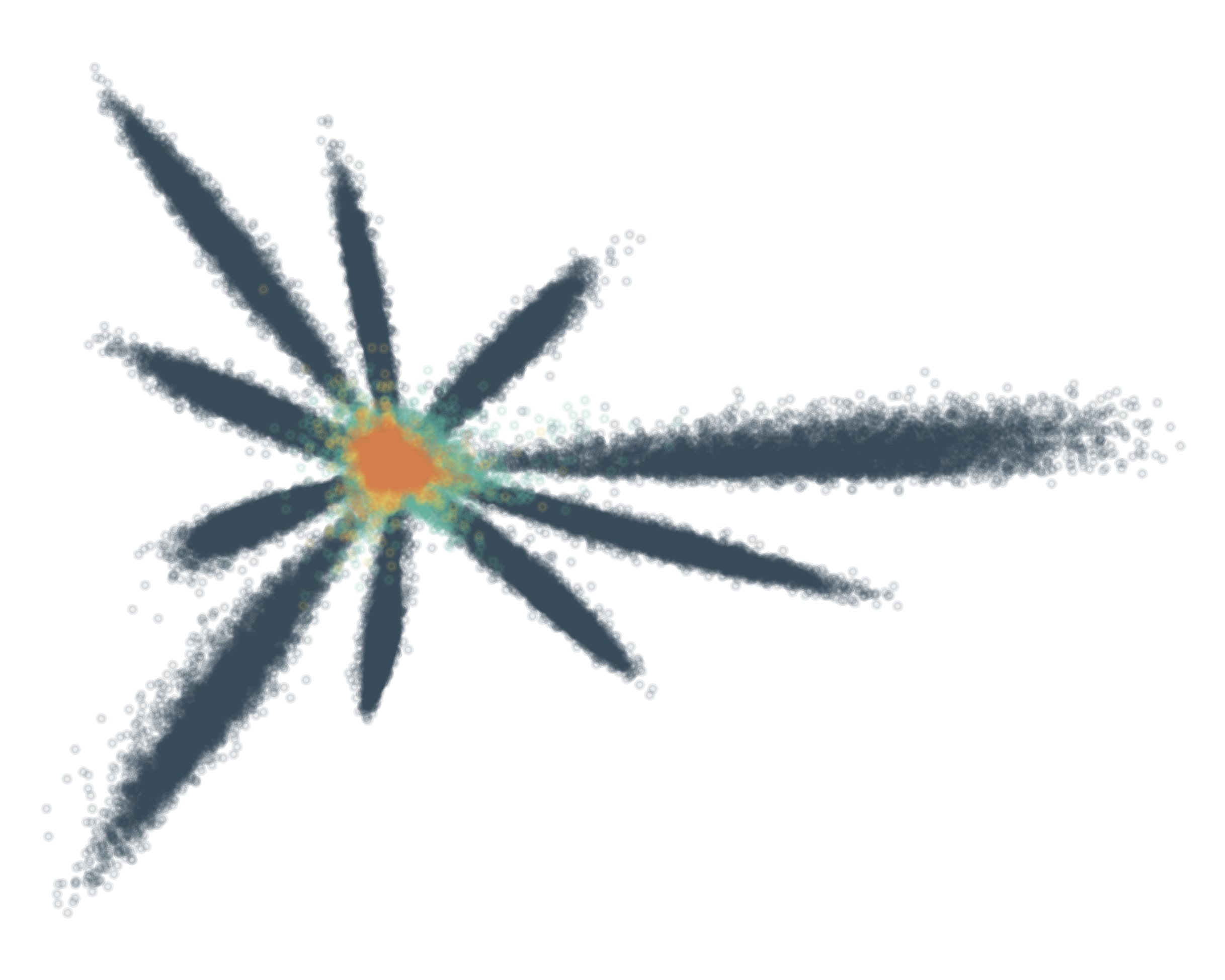}
				\label{fig:intro_softmax}
			\end{minipage}%
		}%
		\subfigure[Prototype Learning]{
			\begin{minipage}[b]{0.23\linewidth}
				\centering
				\includegraphics[width=\linewidth]{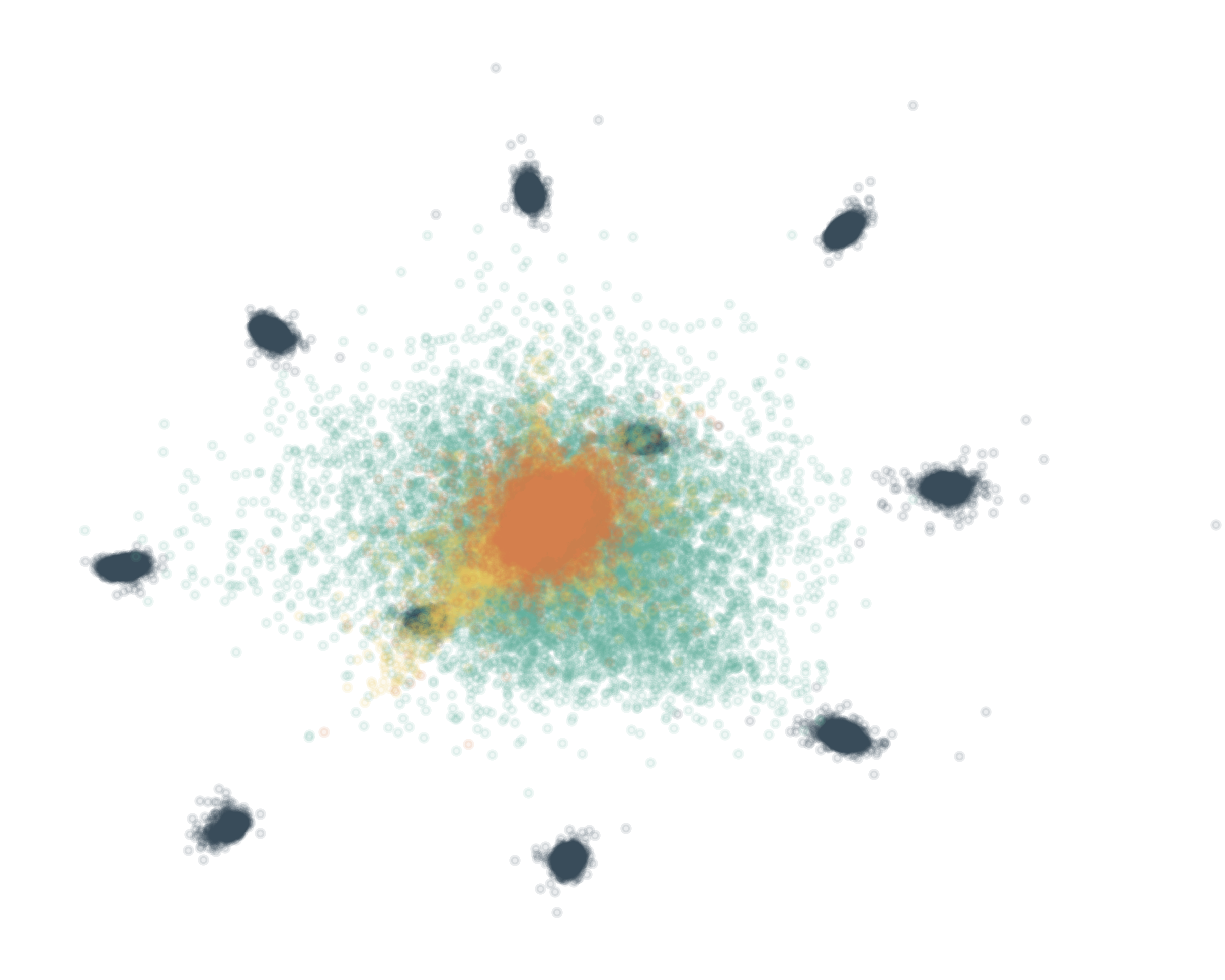}
				\label{fig:intro_GCPL}
			\end{minipage}%
		}%
		\subfigure[ARPL]{
			\begin{minipage}[b]{0.23\linewidth}
				\centering
				\includegraphics[width=\linewidth]{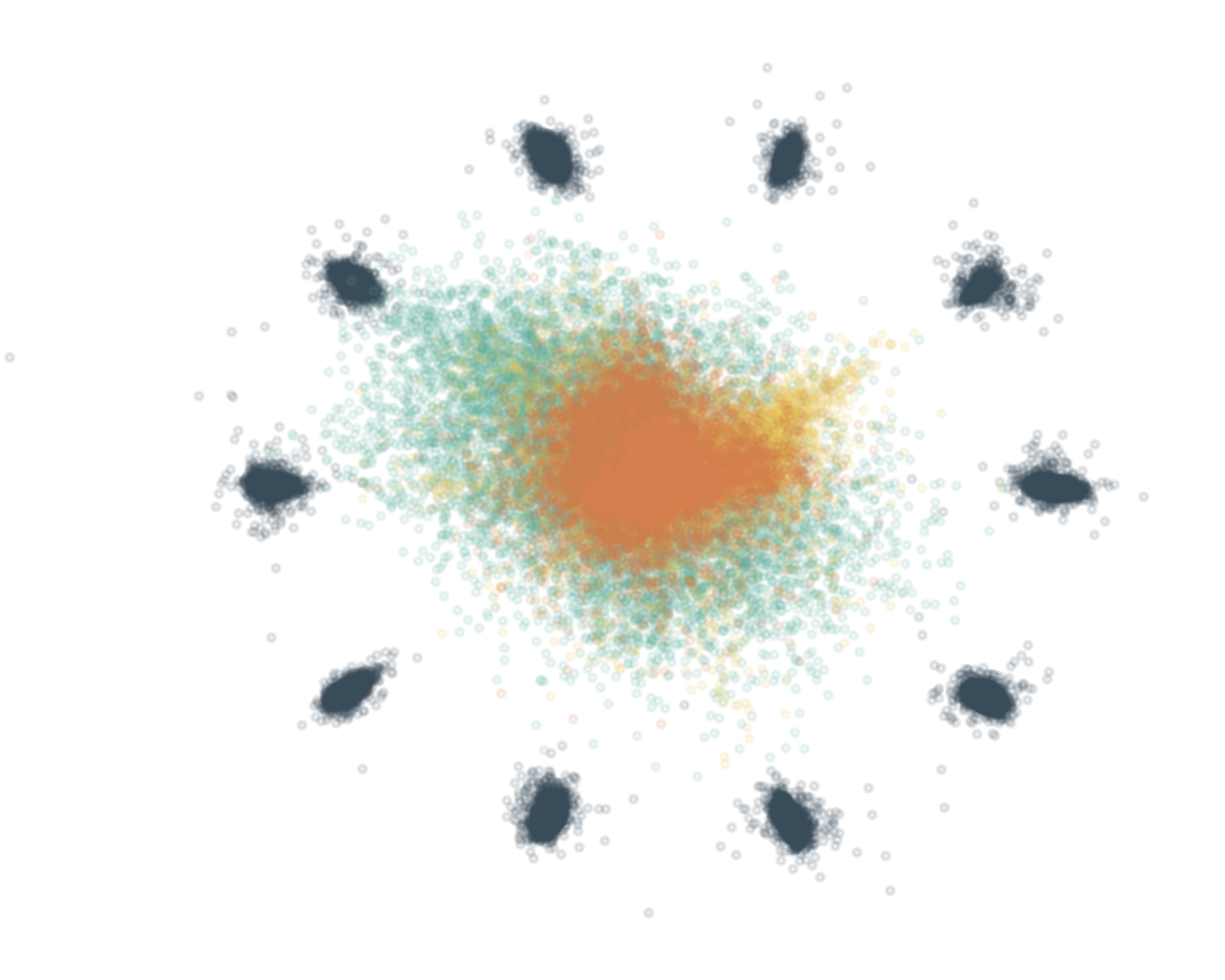}
				\label{fig:intro_ARPL}
			\end{minipage}%
		}%
		\caption{(a): Image-space has infinite open space, but deep responses of most unknowns distribute finite low magnitude areas of deep space \cite{dhamija2018reducing}. (b)-(d): \textbf{LENET++ DEEP RESPONSES TO KNOWNS AND UNKNOWNS.} MNIST (blue) is used for known training, and KMNIST (green), SVHN (yellow) and CIFAR-100 (orange) are used for open set evaluation, where their similarity with MNIST gradually decreased. The network in (b) was trained by Softmax, while the networks in (c) and (d) are trained with Prototype Learning \cite{yang2018robust} and our novel Adversarial Reciprocal Prototype Learning (ARPL). 
			This paper addresses how to improve recognition by reducing the overlap between the deep features from known samples and the features from different unknown samples. In an application, a score threshold should be chosen to optimally separate various unknown from known samples. Unfortunately, such a threshold is difficult to find for either (b) or (c). A better separation is achievable with (d).
		}
		\label{fig:intro}
	\end{figure*}

	To model the potential unknown space without corresponding samples, a novel concept, the \emph{Reciprocal Point} is proposed in this article. 
	Consider a straightforward case with only one known class such as \emph{cat} in Fig.~\ref{fig:RP}.
	\emph{How to identify a cat?}
	Most classification methods aim to learn \emph{"what is a \emph{cat}?"}, resulting in seeing only one spot in the whole problem space. In contrast, \emph{Reciprocal Points}, as potential representative features of \emph{non-cats}, identify cats by otherness. Here a reciprocal point is typically adverse to the prototype of a known class. All these \emph{Reciprocal Points} constitute an instantiated representation of the latent unexploited \emph{extraclass} space, which can potentially used to reduce uncertainty when solving the problem \emph{"what is a \emph{cat}?"} in an OSR setting.

	For the known category \emph{cat}, most unknown samples obviously belong to the space of non-\emph{cats} and their features should be more similar to the representation of non-\emph{cats}, which means that the corresponding unknown information is more implicit in each non-\emph{cats} embedding space.
	Therefore, a novel classification framework is proposed based on the confrontation between multiple known classes and their reciprocal points. It aims to enlarge the distance between the embedding features of the target class and the corresponding reciprocal points, as shown in Fig.~\ref{fig:overview}(a).
	We also formulate the open set risk from the perspective of multiclass integration.
	To reduce the open space risk for each known class from potentially unknown data, a novel adversarial margin constraint term is proposed to limit the extraclass embedding space in a bounded range by binding the target class and its reciprocal point.
	Furthermore, each known class belongs to the extra-class space of other classes.
	When multiple classes interact with each other in the training stage, all the known classes are not only pushed to the periphery of the space by the corresponding reciprocal points for classification, but also pulled into a certain bounded range by other reciprocal points with adversarial margin constraints. 
	Finally, as shown in Fig.~\ref{fig:overview}(b), all known classes are distributed around the periphery of the bounded embedding space, and the unknown samples are limited to the internal bounded space. 
	The bounded constraint prevents the neural network from generating arbitrarily high confidence for unknown samples. 
	Although only known samples are available during the training stage, the interval between known and unknown classes is separated indirectly by reciprocal points.

	To estimate the unknown distribution from the open space, a novel \emph{Instantiated Adversarial Enhancement} mechanism is proposed to generate confusing training samples, to enhance the model distinguishability of known and unknown classes. 
	Unlike the common Generative Adversarial Network (GAN) \cite{goodfellow2014generative}, the proposed method involves an additional adversarial strategy between the discriminator and classifier: 
	On the one hand, the generated samples should deceive the discriminator so that it judge them to be known samples; 
	on the other hand, the classifier's responses for the generated samples are encouraged to be close to each reciprocal point, as illustrated in Fig.~\ref{fig:overview}(c). This means that the generated samples should be as close to the open space of the classifier's embedding space as possible. 
	Finally, the generator, discriminator and classifier are trained jointly to achieve equilibrium.
	More diverse and confusing samples are generated during this process to promote the classifier to filter out most samples that are significantly different from the known samples.
	In addition, an auxiliary batch normalization module and a focusing training mechanism are developed to prevent the classifier from making confusing predictions due to the diverse generated samples.
	
	\begin{figure}[!tb]
		\centering
		\includegraphics[width=\linewidth]{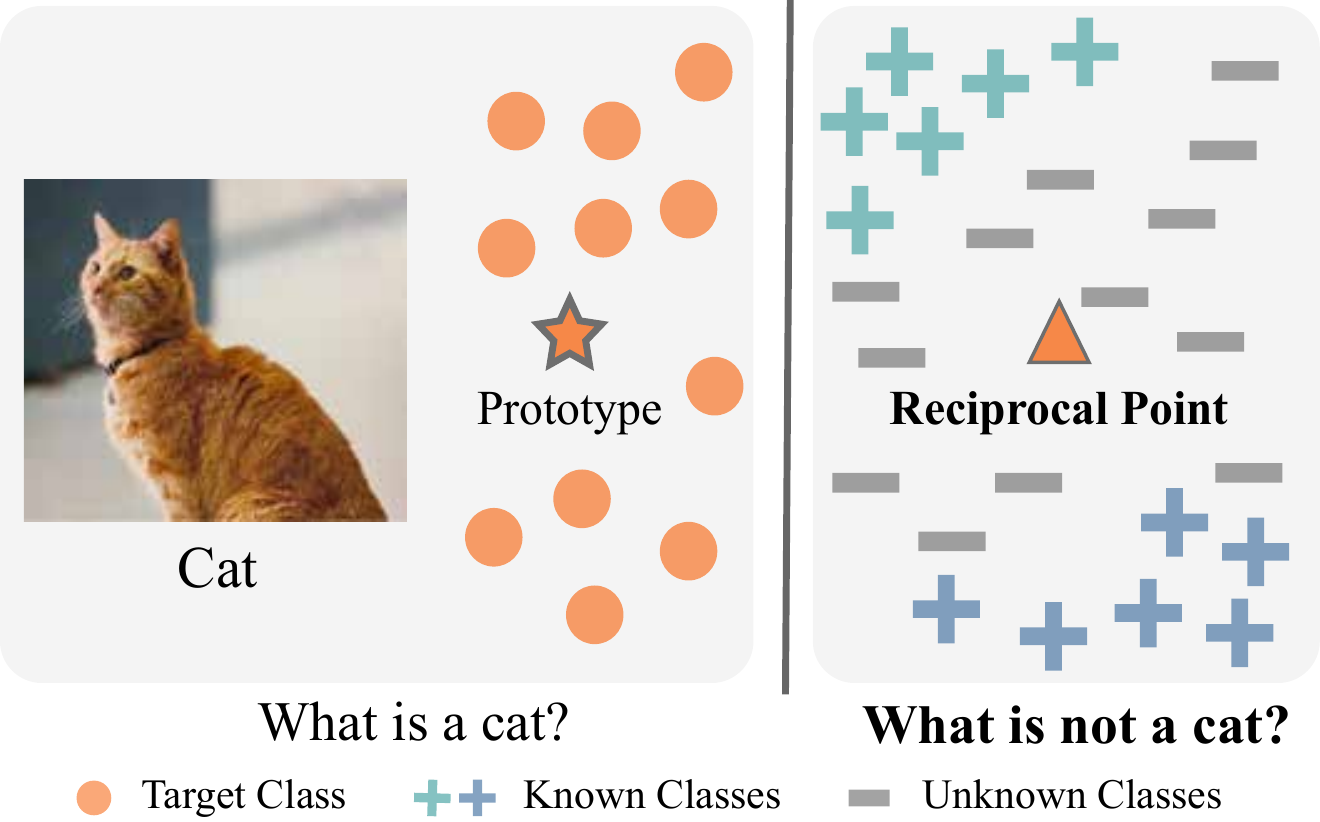}
		\caption{How to identify a cat in the OSR setting? Most methods focus on learning the potentially representative features of cats as prototypes. In contrast, \emph{Reciprocal Points}, as potentially representative features of \emph{non-cats}, identify the cat by otherness. Here these \emph{Reciprocal Points} constitute an instantiated representation of the \emph{extra-class} space, which can potentially be used to reduce the uncertainty when solving the OSR problem.}
		\label{fig:RP}
	\end{figure}
	
	Our contributions are summarized as follows: (1) The open space risk is formulated from the perspective of multiclass integration, by introducing a novel concept, the \emph{Reciprocal Point}, to model the latent open space for each known class in the feature space. (2) Based on reciprocal points with an adversarial margin constraint among multiple known categories, a classification framework is introduced to reduce the empirical classification risk and the open space risk. The rationality of the adversarial margin constraint is theoretically guaranteed by Theorem \ref{theorem:multi}. (3) To estimate the unknown distribution from the open space, in particular the indistinguishable part with known categories, a novel instantiated adversarial enhancement is designed to generate more diverse confusing training samples from the confrontation between the known data and reciprocal points.
	
	This study extends our ECCV spotlight paper \cite{chen_2020_ECCV} in several respects. 1) We develop a novel instantiated adversarial training strategy to enhance the model distinguishability for known and unknown classes by generating confusing training samples (Section \ref{sec:gan}). Experiments are conducted on several datasets to prove the effectiveness of the proposed method.
	2) We improve the preliminary method by incorporating the cosine of the angle to measure the distance between the known classes and their reciprocal points, which has intrinsic consistency with the classification loss (Section \ref{sec:classification}) and brings satisfactory performance improvement.
	3) The adversarial margin constraint is proposed to construct a more elastic bounded space for multiclass adversarial fusion, to learn a more discriminative feature space to identify various unknown distributions. We present theoretical analysis to prove its rationality in
	(Section \ref{sec:risk}).
	4) More qualitative and quantitative experiments are conducted to evaluate the effectiveness of the method, including the following: 
	(a) A more comprehensive metric, Open Set Classification Rate \cite{dhamija2018reducing}, is developed by considering both the distinction between known and unknown classes and the accuracy of known classes. It is more in line with the essence of open set recognition and is used to evaluate different algorithms in Section \ref{EXP:OSR}.
	(b) We add experiments on the out-of-distribution detection task in Section \ref{EXP:OOD}, and more state-of-the-art algorithms for open set recognition for comparison. (c) More visualizations illustrations and correlations are shown to enable better understanding of the embedding feature space for near-to-far unknown samples in Section \ref{EXP:AS}.
	
	\section{Related Work}
	
	\subsection{Open Set Recognition}
	Inspired by a classifier with rejection option \cite{bartlett2008classification, da2014learning, yuan2010classification}, Scheirer \emph{et al.} \cite{scheirer2012toward} defined OSR problem for the first time and proposed a base framework to perform training and evaluation. 
	In recent years, OSR has been surprisingly overlooked, though it has more practical value than the common closed set setting. 
	The few works on this topic can be broadly classified into two categories: discriminative models and generative models.
	
	\textit{Discriminative Methods}. 
	Before the deep learning era, several OSR works utilizing traditional machine learning methods were proposed.
	For example, Scheirer \emph{et al.} \cite{scheirer2014probability} and Jain \emph{et al.} \cite{jain2014multi} considered a distribution of decision scores for unknown detection based on extreme value models with the Support Vector Machines (SVMs).
	Rudd \emph{et al.} proposed extreme value machines \cite{rudd2018extreme} which modeled class-inclusion probabilities with an extreme value theory (EVT) based density function.
	Junior \emph{et al.} \cite{junior2017nearest} proposed an open set nearest neighbor method, which identified any test sample having low similarity to known samples.
	The similarity scores were calculated using the ratio of the distances between the nearest neighbors. Zhang \emph{et al.} \cite{zhang2017sparse} proposed a sparse-representation-based OSR method, which also used EVT to identify unknown samples by residual errors.
	Note that these methods usually do not scale well without careful feature engineering.
	Recently, deep neural networks (DNNs) were also introduced to the OSR task by Bendale \emph{et al.}  \cite{bendale2016towards}. 
	They proved that the threshold on the softmax probability does not yield a robust model for OSR.
	Openmax \cite{bendale2016towards} was then proposed to detect unknown classes by modeling the distance of the activation vectors.
	Shu \emph{et al.} \cite{shu2017doc} proposed a K-sigmoid activation-based method, which enabled the end-to-end training by eliminating outlier detectors outside the network. In these works, the sigmoid function did not have the \emph{compact abating property} \cite{scheirer2014probability}.
	This property may be activated by adding an infinitely distant input from all the training data, and thus its open space risk is not bounded \cite{yoshihashi2019classification}.
	
	\textit{Generative Methods}. 
	Unlike discriminative models, generator approaches generate unknown or known samples using GANs \cite{goodfellow2014generative}, autoencoders \cite{sun2020conditional} and flow-based Models \cite{zhang2020hybrid} to help the classifier to learn the decision boundary between known and unknown samples.
	Ge \emph{et al.} \cite{ge2017generative} proposed G-Openmax, a direct extension of Openmax, using generative models to synthesize unknown samples to train the network.
	Similar to the idea in \cite{zhang2017sparse}, Yoshihashi \emph{et al.} \cite{yoshihashi2019classification} proposed the CROSR model, which combined the supervised learned prediction and unsupervised reconstructive latent representation to redistribute the probability distribution.
	\cite{oza2019c2ae} proposed the C2AE model for OSR, using class conditional autoencoders to obtain the decision boundary from the reconstruction errors by EVT. Xin \emph{et al.} \cite{sun2020conditional} provided a conditional Gaussian distribution learning for the Variational Auto-Encoder (VAE) to detect unknowns and classify known samples by forcing different latent features to approximate different Gaussian models.
	Zhang \emph{et al.} \cite{zhang2020hybrid} proposed a composed of classifier and a flow-based density estimator into a joint embedding space.
	However, these methods did not consider the deep distribution of unknown classes in learners, resulting in potential open space risk.
	
	\subsection{Out-of-Distribution Detection}
	The OSR is naturally related to some other problem settings such as out-of-distribution (OOD) detection \cite{hendrycks17baseline}, outlier detection \cite{rozsa2017adversarial}, and novel detection \cite{perera2019ocgan}, etc.
	Considering the safety of AI systems, the detection of OOD examples was first introduced by Hendrycks \emph{et al.} in \cite{hendrycks17baseline}. Here 
	OOD detection is the detection of samples that do not belong to the training set but may appear during testing \cite{hendrycks17baseline}. 
	Hendrycks \emph{et al.} \cite{hendrycks17baseline} demonstrated that anomalous samples have a lower maximum softmax probability than in-distribution samples. 
	Liang \emph{et al.} \cite{liang2017enhancing} proposed ODIN to allow more effective detection by using temperature scaling and adding small perturbations to the input. 
	Lee \emph{et al.} \cite{lee2017training} utilized generative models to generate the most effective samples from the OOD samples and derived a new OOD score from this branch.
	Hendrycks \emph{et al.} \cite{hendrycks2018deep} proposed outiler exposure by using an auxiliary dataset to teach the network better representations for anomaly detection. 
	OOD detection is similar to the rejection of unknown classes in OSR, because they are both study the separation of in-distribution (known) and out-of-distribution (unknown) samples \cite{hendrycks17baseline,scheirer2012toward} and do not require the discriminator power for known classes.

	\subsection{Prototype Learning}
	A prototype is an average or best exemplar of a category, thus can provide a concise representation for the entire category of instances \cite{kuncheva1998nearest}. 
	The best-known prototype learning method is k-nearest-neighbors (KNN). In \cite{liu2001evaluation}, learning vector quantization (LVQ) was proposed to save the storage space and improve the computational efficiency of KNN. In most previous works, prototypes are learned by optimizing the self-defined object functions \cite{sato1998formulation}. 
	Recently, some methods have also combined prototype learning with a probabilistic model and a neural network for classification tasks. 
	Under the framework of neural networks, prototypes are learnable representations in the form of one or more latent vectors per class. 
	In \cite{bonilla2012discriminative}, the authors represented the input instance as a K-dimensional vector, modeled each component as a mixture of probabilities, and finally applied a probabilistic model to parameterize K-prototype patterns through the likelihood maximization.
	Wen \emph{et al.} \cite{wen2016discriminative} proposed a center loss to learn the centers of the deep features of each identity and used the centers to reduce intra-class variance. 
	Yang \emph{et al.} \cite{yang2018robust,Prototype} proposed the Generalized Convolutional Prototype Learning (GCPL) with a prototype loss, which was used as a regularization method to improve the intra-class compactness of the feature representation.
	For the OSR problem, the prototype helps to reduce intraclass distance of the known classes, but it ignores the potential characteristics of the unknown data, resulting in less effectiveness in reducing the open space risk.

	\begin{figure*}[!tb]
		\centering
		\includegraphics[width=0.98\linewidth]{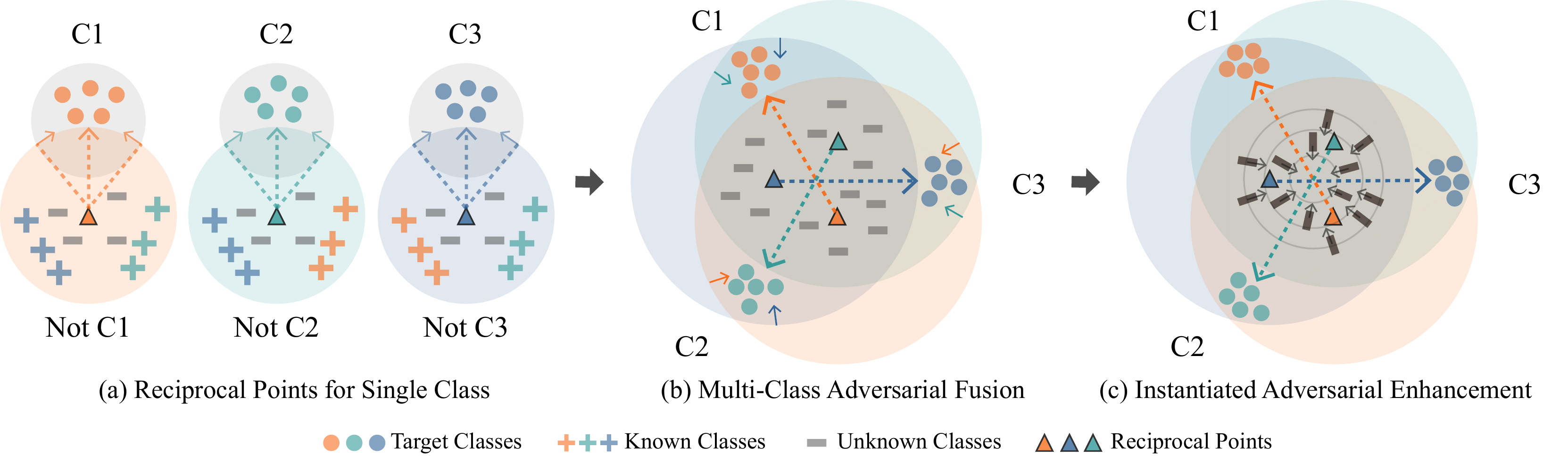}
		\caption{An overview of the proposed \textbf{Adversarial Reciprocal Point Learning} (ARPL) approach for open set recognition. 
			(a) \emph{Reciprocal Points for Single Class} promote each known class far away from their reciprocal points. 
			(b) \emph{Multi-Class Adversarial Fusion} induces the confrontation between multi-category bounded spaces constructed by reciprocal points. As a result, the known classes are pushed to the periphery of the feature space, and the unknown classes are limited in the bounded space.
			(c) \emph{Instantiated Adversarial Enhancement} generates more valid and more diverse confusing samples to promote the reliability of the classifier.
		}
		\label{fig:overview}
	\end{figure*}
	
	\section{Adversarial Reciprocal Point Learning}
	
	\subsection{Problem Definition}
	\label{notion}
	\noindent
	Given a set of $n$ labeled samples $\mathcal{D}_\mathcal{L}=\lbrace (x_1, y_1), \dots,$ $ (x_n, y_n) \rbrace$ with $N$ known classes, where $y_i \in \lbrace 1, \dots, N \rbrace$ is the label of $x_i$, and a larger amount of test data $\mathcal{D}_\mathcal{T}=\lbrace t_1, \dots, t_{u} \rbrace$ where the label of $t_i$ belongs to $\lbrace 1, \dots, N \rbrace \cup \lbrace N+1, \dots, N+U \rbrace $ and $U$ is the number of unknown classes in realistic scenarios, the deep embedding space of category $k$ is denoted by $\mathcal{S}_k$ and its corresponding \textbf{open space} is denoted as $\mathcal{O}_k$.
	To formalize and manage the open space risk effectively, $\mathcal{O}_k$ is separated into two subspaces: the positive open space from other known classes as $\mathcal{O}^{pos}_k$ and the remaining infinite unknown space as the negative open space $\mathcal{O}^{neg}_k$. That is, $\mathcal{O}_k = \mathcal{O}^{pos}_k \cup \mathcal{O}^{neg}_k$.
	
	In our method, the samples $\mathcal{D}^k_\mathcal{L} \in \mathcal{S}_k$ from category $k$, samples $\mathcal{D}^{\neq k}_\mathcal{L} \in \mathcal{O}^{pos}_k$ from other known classes, and samples $\mathcal{D}_\mathcal{U} \in \mathcal{O}^{neg}_k$ from $\mathbb{R}^d$ other than $\mathcal{D}_\mathcal{L}$, are defined as the positive training data, the negative training data and the potential unknown data respectively. The binary measurable prediction function $\psi_k: \mathbb{R}^d \mapsto \lbrace 0, 1 \rbrace $ is used to map the embedding $x$ to the label $k$. For the 1-class OSR problem, the overall goal is to optimize a discriminative binary function $\psi_k$ by minimizing the expected error $\mathcal{R}^k$:
	\begin{equation}
		\mathop{\arg\min}_{\psi_k} \lbrace \mathcal{R}^k \mid         
		\mathcal{R}_\epsilon(\psi_k, \mathcal{S}_k \cup \mathcal{O}^{pos}_k) + \alpha \cdot
		\mathcal{R}_o(\psi_k, \mathcal{O}^{neg}_k) \rbrace,
		\label{Eqn:1}
	\end{equation}
	where $\alpha$ is a positive regularization parameter, $\mathcal{R}_\epsilon$ is the empirical classification risk on the known data, and $\mathcal{R}_o$ is the \textbf{Open Space Risk} \cite{scheirer2012toward} that is used to measure the uncertainty of labeling the unknown samples as the known or unknown class. This is further formulated as a nonzero integral function on space $\mathcal{O}^{neg}_k$:
	\begin{equation}
		\mathcal{R}_o(\psi_k, \mathcal{O}^{neg}_k) = \frac{\int_{\mathcal{O}^{neg}_k} \ \psi_k(x) dx}{\int_{\mathcal{S}_k \cup \mathcal{O}_k} \ \psi_k(x) dx}.
		\label{Eqn:risk}
	\end{equation}
	The more often the open space $\mathcal{O}^{neg}_k$ is labeled as positive , the greater the open space risk is. 
	
	In the multiclass setting, the OSR problem is identified by integrating multiple binary classification tasks (\textit{one} \textit{vs.} \textit{rest}) (as shown in Fig.~\ref{fig:overview}). By summarizing the expected risk in Eq.\ \eqref{Eqn:1} among all known categories, i.e.,\ $\mathop{\sum}\nolimits_{k=1}^N \mathcal{R}^k$, we obtain
	\begin{equation}
		\mathop{\sum}_{k=1}^N \ \ \mathcal{R}_\epsilon(\psi_k, \ \mathcal{S}_k \cup \mathcal{O}^{pos}_k) + 
		\alpha \cdot \mathop{\sum}_{k=1}^N \ \ \mathcal{R}_o(\psi_k, \ \mathcal{O}^{neg}_k).
		\label{Eqn:3}
	\end{equation}
	Minimizing the left side of Eq.\ \eqref{Eqn:3} is equivalent to training multiple-binary classifiers, producing a multiclass prediction function $f=\odot(\psi_1, \psi_2, \dots, \psi_N)$ for $N$-category classification, where $\odot(\cdot)$ is the intergration operation.
	Hereafter, Eq.\ \eqref{Eqn:3} is formulated as:
	\begin{equation}
		\mathop{\arg\min}_{f \in \mathcal{H}} \ \lbrace 
		\mathcal{R}_\epsilon(f, \mathcal{D}_L) + \alpha \cdot 
		\mathop{\sum}\nolimits_{k=1}^N\mathcal{R}_o(f, \mathcal{D}_U) \rbrace
		\label{Eqn:2}
	\end{equation}
	where $f: \mathbb{R}^d \mapsto \mathbb{N}$ is a measurable multiclass recognition function, $\mathcal{D}_L$ is the set of labeled data used during the training phase, and $\mathcal{D}_U$ is the potentially unknown data.
	According to Eq.\ \eqref{Eqn:2}, solving the OSR problem is equivalent to minimizing the combination of the empirical classification risk on labeled known data and the open space risk on potential unknown data simultaneously, over the space of allowable recognition functions. This makes embedding function more distinguishable between known and unknown spaces. 
	
	\subsection{Reciprocal Points for Classification}
	\label{sec:classification}
	The \textbf{reciprocal point} $ \mathcal{P}^k$ of category $k$ is regarded as the latent representation of the subdataset $\mathcal{D}^{\neq k}_\mathcal{L} \cup \mathcal{D}_\mathcal{U}$. 
	Hence, the samples of $\mathcal{O}_k$ should be closer to the reciprocal point $ \mathcal{P}^k$ than the samples of $\mathcal{S}_k$, which is formulated as:
	\begin{equation}
		\max(\zeta(\mathcal{D}^{\neq k}_L \cup \mathcal{D}_U, \ \mathcal{P}^k))
		\ \leq \ d, \ \forall d \in \zeta(\mathcal{D}^k_L, \ \mathcal{P}^k), 
		\label{Eqn:r-point}
	\end{equation} 
	where $\zeta(\cdot, \ \cdot)$ calculates the set of distances of all samples between two sets. Based on Eq. \eqref{Eqn:r-point}, the samples can be classified by the opposition between the reciprocal points and the corresponding known classes.
	
	Specifically, the reciprocal point of a class is represented by an $m$-dimensional representation, and can be optimized by an deep embedding function $\mathcal{C}$ with learnable parameters $\theta$.
	Given sample $x$ and reciprocal point $\mathcal{P}_k$, their distance $d(\mathcal{C}(x), \mathcal{P}^k)$ is calculated by combining the Euclidean distance $d_e$ and dot product $d_d$:
	\begin{equation}
		\begin{aligned}
			d_e(\mathcal{C}(x), \mathcal{P}^k) &= \frac{1}{m} \cdot { \lVert \mathcal{C}(x) - \mathcal{P}^k \rVert}^2_2, \\
			d_d(\mathcal{C}(x), \mathcal{P}^k) &= \mathcal{C}(x) \cdot \mathcal{P}^k, \\
			d_{\ }(\mathcal{C}(x), \mathcal{P}^k) &= d_e(\mathcal{C}(x), p^k_i) - d_d(\mathcal{C}(x), \mathcal{P}^k).
			\label{Eqn:distance}
		\end{aligned}
	\end{equation} 
	Each known class is opposite to its reciprocal point in terms of both spatial position and angle direction.
	The combination of the Euclidean similarity and the dot product is capable of better evaluating the similarity between the embedding features of the known classes and their reciprocal points. 
	
	Based on the proposed distance metrics, our framework estimates the otherness between the embedding feature $\mathcal{C}(x)$ and the reciprocal points of all known classes to determine which category it belongs to. 
	Following the nature of reciprocal points, the probability of sample $x$ belonging to category $k$ is proportional to the otherness between $\mathcal{C}(x)$ and the reciprocal point $\mathcal{P}^k$, where a greater distance between $\mathcal{C}(x)$ and $ \mathcal{P}^k$ leads to assign the sample $x$ to label $k$ with a larger probability.
	According to the sum-to-one property, the final classification probability is normalized with the softmax function:
	\begin{equation}
		\label{Eqn:prob}
		p(y=k|x, \mathcal{C}, \mathcal{P}) = \frac{e^{\gamma d(\mathcal{C}(x), \  \mathcal{P}^k)}}{\sum\nolimits_{i=1}^{N}{e^{\gamma d(\mathcal{C}(x), \ \mathcal{P}^i)}}},
	\end{equation}
	where $\gamma$ is a hyperparameter that controls the hardness of the distance-probability conversion.
	The learning of $\theta$ is achieved by minimizing the reciprocal points classification loss based the negative log-probability of the true class $k$:
	\begin{equation}
		\mathcal{L}_c(x; \theta, \mathcal{P}) = - \ log \ p(y=k|x, \mathcal{C}, \mathcal{P}).
		\label{Eqn:loss}
	\end{equation}
	Through minimizing Eq. \eqref{Eqn:loss} which corresponds to $ \mathcal{R}_\epsilon(f, \mathcal{D}_L) $ in Eq. \eqref{Eqn:2}, the reciprocal points classification loss reduces the empirical classification risk through the reciprocal points.
	
	In addition to classifying the known classes, an advantage of minimizing Eq. \eqref{Eqn:loss} is to separate known and unknown spaces by maximizing the distance between the reciprocal points of the category and its corresponding training samples as follows:
	\begin{equation}
		\mathop{\arg\max}_{f \in \mathcal{H}} \ \lbrace 
		\zeta(\mathcal{D}_L^k, \mathcal{P}^k) \rbrace.
		\label{Eqn:maximize}
	\end{equation}
	
	Although Eq. \eqref{Eqn:loss} and Eq. \eqref{Eqn:maximize} facilitate the maximization of the interval between the closed space $\mathcal{S}_k$ and the center of the open space $\mathcal{O}_k$, $\mathcal{O}_k$ is not constrained in Eq. \eqref{Eqn:loss}. Hence, $\mathcal{S}_k$ and $ \mathcal{O}_k $ may have an inestimable overlap (as shown in Fig.~\ref{fig:ARPL_0}), meaning that the open space risk still exists.
	
	\subsection{Adversarial Margin Constraint}
	\label{sec:risk}
	
	To reduce the open space risk $\mathcal{R}_o(f, \mathcal{D}_U) $ in Eq. \eqref{Eqn:2}, a novel \emph{Adversarial Margin Constraint} (AMC) is proposed to constrain the open space, where each particular category $k$ contains the positive open space $\mathcal{O}_k^{pos}$ and the infinite negative open space $\mathcal{O}_k^{neg}$. 
	For multiclass OSR scenarios, multiple class-wise open spaces are united into a global open space $\mathcal{O}_G$:
	\begin{equation}
		\mathcal{O}_G = \bigcap\nolimits^{N}_{k=1} (\mathcal{O}_k^{pos} \cup \mathcal{O}_k^{neg} ),
		\label{Eqn:OG}
	\end{equation}
	where the total open space risk can be restricted by limiting the open space risk for each known class. 
	
	To separate $\mathcal{S}_k$ and $\mathcal{O}_k$ as much as possible, 
	the open space $\mathcal{O}_k$ must be restricted so that the open set space can be estimated. We aim to reduce the open space risk of each known class by limiting the open space $\mathcal{O}_k$ in a bounded range. This has the effect facilitating an increase in the maximum value of the distance between the negative/unknown data and reciprocal points less than $R$. Hence, the following formula is established: 
	\begin{equation}
		\max(\zeta(\mathcal{D}^{\neq k}_L \cup \mathcal{D}_U, \ \mathcal{P}^k)) 
		\ \leq \ \
		R.
		\label{Eqn:reg}
	\end{equation} 
	Clearly, it is almost impossible to manage the open space risk by restricting the open space, because the open space contains a large number of unknown samples $\mathcal{D}_U$. 
	However, considering that spaces $\mathcal{S}_k$ and $\mathcal{O}_k$ are complementary to each other, the open space risk can be bounded indirectly by constraining the distance between the samples from $\mathcal{S}_k$ and the reciprocal points $\mathcal{P}^k$ to be smaller than $R$ as follows:
	\begin{equation}
		\mathcal{L}_o(x; \theta, \mathcal{P}^k, R^k) = max(d_e(\mathcal{C}(x), \mathcal{P}^k) - R, 0),
		\label{Eqn:regularization}
	\end{equation}
	where $R$ is a learnable margin and only the Euclidean distance is used instead of a directional metric to obtain a larger range of non-$k$ samples.
	Specifically, minimizing Eq. \eqref{Eqn:regularization} by the classification loss $\mathcal{L}_c$ is equivalent to making $\zeta(\mathcal{D}^{\neq k}_L \cup \mathcal{D}_U, \ \mathcal{P}^k)$ in Eq. \eqref{Eqn:r-point} as smaller as possible compared with $R$. Here we use a theorem to better illustrate the advantage of our method.
	
	\begin{theorem} \label{theorem:multi}
		For a neural network whose logit layer is based on reciprocal points, $x \in \mathcal{D}_L^k$, $\mathcal{L}_c$ and $\mathcal{L}_o$ are minimized simultaneously if and only if $\max(\zeta(\mathcal{D}^{\neq k}_L, \ \mathcal{P}^k)) \ \leq \ \ R$.
	\end{theorem}
	
	\begin{proof}
		
		We give a proof by contradiction. 
		\begin{itemize}
			\item For $ x \in \mathcal{D}_L^{k}$, we assume that there is a sample of category $t$: $s \in \mathcal{D}^{t}$, where $t \neq k$, and that $\zeta(s, \ \mathcal{P}^k) > \ R$.
			\item For such a sample, we can make the following inference:
			\begin{itemize}
				\item First, minimizing $\mathcal{L}_c$ maximizes the distance between each category $k$ and its $\mathcal{P}^k$.
				\item Second, the loss $\mathcal{L}_o$ is minimized when $\forall k \in \{1,\dots,N\}, \max(\zeta(\mathcal{D}^{k}_L, \ \mathcal{P}^k)) \leq R$ in Eq. \eqref{Eqn:regularization}. The loss $\mathcal{L}_o$ is minimized, so $\zeta(s, \ \mathcal{P}^t) \leq \ R$.
				\item Then, $\zeta(s, \ \mathcal{P}^t) < \zeta(s, \ \mathcal{P}^k)$. Samples $s$ is classified into category $k$ in Eq. \eqref{Eqn:prob}, increasing the loss $\mathcal{L}_c$.
			\end{itemize}
			\item These consequences contradict what we have just assumed that sample $s$ belongs to category $t$.
			\item As a result, for $x \in \mathcal{D}_L^k$, $\mathcal{L}_c$ and $\mathcal{L}_o$ are minimized simultaneously if and only if $\max(\zeta(\mathcal{D}^{\neq k}_L, \ \mathcal{P}^k)) \ \leq \ \ R$.
		\end{itemize}
	\end{proof}
	
	\begin{algorithm}[!tb]
		\fontsize{8}{8}
		\renewcommand{\algorithmicrequire}{\textbf{Input:}}
		\renewcommand{\algorithmicensure}{\textbf{Output:}}
		\caption{The adversarial reciprocal point learning algorithm.}
		\label{alg:1}
		\begin{algorithmic}[1]
			\REQUIRE Training data $\{ x_i \}$. Initialized parameters $\theta$ in the convolutional layers, and parameters $\mathcal{P}$ and $R$ in the loss layers, respectively. Hyperparameter $ \lambda, \gamma $ and learning rate $ \mu $. The number of iteration $t \leftarrow 0$.
			\ENSURE The parameters $\theta$, $\mathcal{P}$ and $R$.
			\STATE while \textit{not converge} do
			\STATE  \qquad $t \leftarrow t + 1$.
			\STATE  \qquad Compute the joint loss by $ \mathcal{L}^t = \mathcal{L}_c^t + \lambda \cdot \mathcal{L}_o^t $.
			\STATE  \qquad Compute the backpropagation error $\frac{\partial \mathcal{L}^t}{\partial x^t}$ for each $i$ by $\frac{\partial \mathcal{L}^t}{\partial x^t} = \frac{\partial \mathcal{L}_c^t}{\partial x^t} + \lambda \cdot \frac{\partial \mathcal{L}_o^t}{\partial x^t}$.
			\STATE  \qquad Update the parameters $\mathcal{P}$ by $\mathcal{P}^{t+1} = \mathcal{P}^{t} - \mu^t \cdot \frac{\partial \mathcal{L}^t}{\partial \mathcal{P}^t} = \mathcal{P}^{t} - \mu^t \cdot (\frac{\partial \mathcal{L}_c^t}{\partial \mathcal{P}^t} + \lambda \cdot \frac{\partial \mathcal{L}_o^t}{\partial \mathcal{P}^t})$.
			\STATE  \qquad Update the parameters $R$ by $R^{t+1} = R^{t} - \mu^t \cdot \frac{\partial \mathcal{L}^t}{\partial R^t} = R^{t} - \lambda \cdot \mu^t \cdot \frac{\partial \mathcal{L}_o^t}{\partial R^t}$.
			\STATE \qquad Update the parameters $\theta$ by $\theta^{t+1} = \theta^{t} - \mu^t \cdot \frac{\partial \mathcal{L}^t}{\partial \theta^t} = \theta^{t} - \mu^t \cdot (\frac{\partial \mathcal{L}_c^t}{\partial {\theta}^t} + \lambda \cdot \frac{\partial \mathcal{L}_o^t}{\partial \theta^t})$.
			\STATE end while
		\end{algorithmic}  
	\end{algorithm}
	
	Theorem \ref{theorem:multi} further indicates that Eq. \eqref{Eqn:reg} can be obtained by limiting the target known class as in Eq. \eqref{Eqn:regularization} with the classification loss $\mathcal{L}_c$. If the target class $k$ and its reciprocal points are contained in a bounded range, the extraclass of class $k$ (including other known classes and potential open space) is also constrained into a bounded space.
	With such multicategory interactions, the known categories are constrained to each other. 
	On the one hand, the former classification loss in Eq. \eqref{Eqn:maximize} can be expected to increase the distance between class $k$ and its reciprocal point $\mathcal{P}^{k}$.
	On the other hand, the class $k$ is bounded by the other reciprocal points $\mathcal{P}^{\neq k}$ as follows:
	\begin{equation}
		\mathop{\arg\min}_{f \in \mathcal{H}} \ \lbrace 
		\max(\{\zeta(\mathcal{D}_L^k, \mathcal{P}^{\neq k}) - R\} \cup \{0\}) \rbrace.
		\label{Eqn:risk2}
	\end{equation}
	Through this adversarial mechanism between Eq. \eqref{Eqn:maximize} and Eq. \eqref{Eqn:risk2}, each known class is pushed to the edge of the finite feature space to the maximum extent, moving each far away from its potential unknown space. 
	
	In addition, considering the bounded space $\mathcal{B}(\mathcal{P}^k, R)$ with the reciprocal points $\mathcal{P}^k$ as centers and $R$ as its corresponding intervals, to separate the known and unknown space, we utilize these bounded spaces to approximate the global unknown space $\mathcal{O}_G$ as closely as possible. As a result, the calculation of the loss in Eq. \eqref{Eqn:regularization} can be viewed as reducing the open space risk $\mathcal{R}_o(f, \mathcal{D}_U)$ in Eq. \eqref{Eqn:2}.
	
	
	\subsection{Learning the Open Set Network}
	
	In adversarial reciprocal points learning, the overall loss function combines Eq. \eqref{Eqn:loss} and Eq. \eqref{Eqn:regularization} to handle the empirical classification risk and the open space risk simultaneously:
	\begin{equation}
		\mathcal{L}(x,y; \theta, \mathcal{P}, R) = \mathcal{L}_c(x;\theta, \mathcal{P}) + \lambda \mathcal{L}_o(x; \theta, \mathcal{P}, R),
		\label{Eqn: final-loss}
	\end{equation}
	where $\lambda$ is the weight of the adversarial open space risk module and $\theta, \mathcal{P}, R$ represent the learnable parameters.
	Alg. \ref{alg:1} summarizes the learning details of the open set network with joint supervision. Some additional explanations are also given here.

	Firstly, we discuss \emph{Unknown Classes for the Neural Network.}
	Based on the principle of maximum entropy, for an unknown sample $x_u$ without any prior information, a well-trained closed set discriminant function tends to assign known labels to $x_u$ with a uniform probability. DNNs usually embed the features of unknown samples into spaces with lower magnitudes rather than random positions in the full space. 
	This phenomenon is also consistent with the observation of \cite{dhamija2018reducing} and the visualization results shown in Fig.~\ref{fig:intro}. 
	For real images space, "\emph{All positive examples are alike; each negative example is negative in its own way}" \cite{scheirer2012toward}.
	However, the various negative examples (unknown) are not provided for training the neural network, so as that these samples obtain lower activation magnitudes from the neural network than positive samples (known).
	Hence, most unknown classes are distributed in low-magnitude area of the deep feature space, as shown in Fig.~\ref{fig:intro} and \ref{fig:retrieval}.
	Furthermore, the learned reciprocal points are also distributed in the low-magnitude areas and push the known classes away from the low-magnitude areas, so as to distinguish the known from the unknown, as shown in Fig.~\ref{fig:embedding}.
	As illustrated in Fig.~\ref{fig:retrieval}, the retrieved images gradually become different as we move from the class center to its reciprocal point.
	The learned reciprocal points and the unknown classes have more similarities in the deep feature space.
	Actually, it could not find the specific realistic sample of the reciprocal point.
	Basically, the differences among a large number of unknown classes are still unknown to the classifier.
	Most unknown classes have great commonality for a classifier, and this part of commonness is "\emph{unknown}".
	
	Second, we discuss \emph{Unknown Classes and Reciprocal Points.}
	Since the global open space $\mathcal{O}_G$ is more aggregated, the open set space is able to be constrained through reciprocal points in the deep embedding space.
	As shown in Fig.~\ref{fig:overview} and Fig.~\ref{fig:retrieval}, learning with Eq. \eqref{Eqn: final-loss} pushes the known spaces to the periphery of $\mathcal{O}_G$ and then separates two spaces as much as possible. As a result, an excellent embedding space structure is formed via adversarial reciprocal point learning (ARPL), which can be used to further divide the known classes and most unknown classes.
	
	\begin{figure}[!tb]
		\centering
		\subfigure[ARPL]{
			\begin{minipage}[t]{0.48\linewidth}
				\centering
				\includegraphics[width=\linewidth]{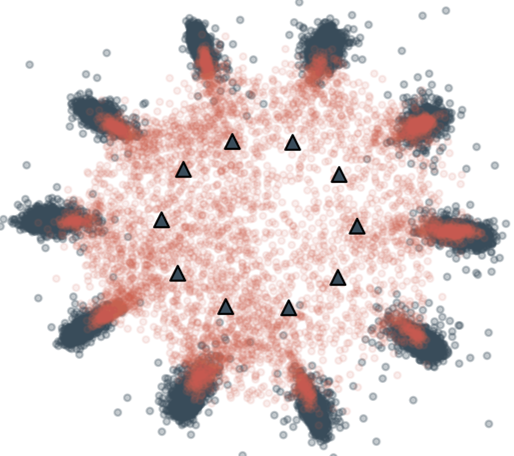}
				\label{fig:osg_ARPL}
			\end{minipage}%
		}
		\subfigure[ARPL+CS]{
			\begin{minipage}[t]{0.48\linewidth}
				\centering
				\includegraphics[width=\linewidth]{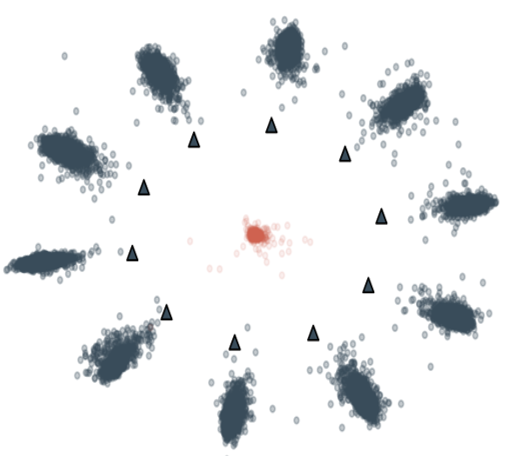}
				\label{fig:osg_ARPL_OS}
			\end{minipage}%
		}
		\caption{The visualization of the feature responses of the neural network to the samples from the confused generator. (a) is trained only by ARPL. (b) is trained by ARPL+CS. The red points represent the embedding of the confusing samples generated through instantiated adversarial enhancement.}
		\label{fig:osg}
	\end{figure}
	
	\section{Instantiated Adversarial Enhancement}
	\label{sec:gan}
	As shown in Fig.~\ref{fig:osg_ARPL}, the ARPL classifier is able to distinguish an unknown distribution without any prior knowledge of the unknown data, but is still vulnerable to the confusing samples generated from a simple generator, even when these samples are quite different from the known categories. To further reduce the open space risk caused by such unknown data, a good solution requires to minimizing the open space from the learned neural networks, with the support of a reasonable optimization strategy for unknown data.
	However, it is a haystack to find valid unknown samples in a real scene.
	Therefore, we further generate the \emph{Confusing Samples} (CS) as unknown data $\mathcal{D}_U$ to improve the discriminability of the classifier for various novel distributions.
	
	\begin{figure}[!tb]
		\centering
		\includegraphics[width=0.95\linewidth]{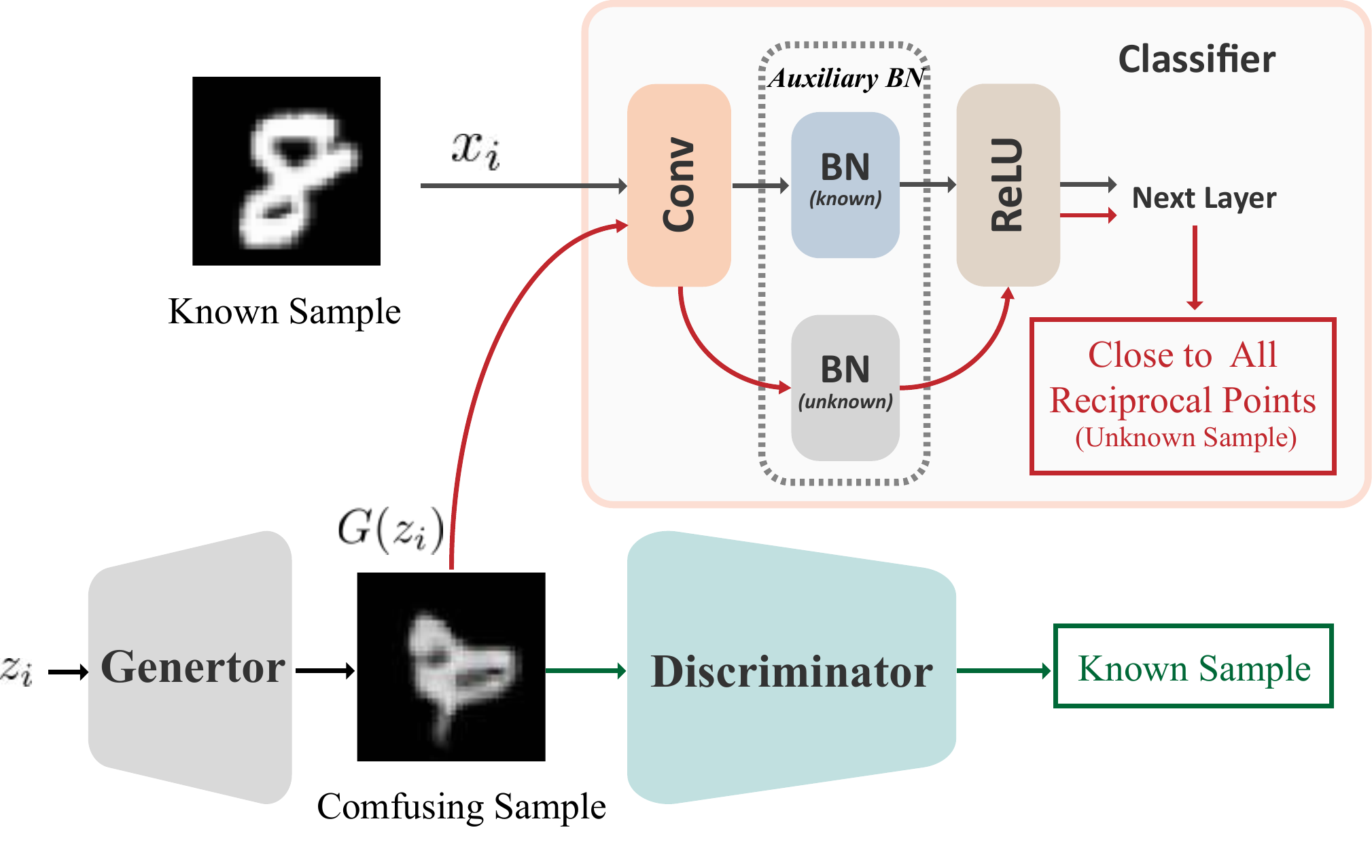}
		\caption{The basic framework of training the confused generator. Here, the generator maps a latent variable $z_i$ to the generated confusing sample $G(z_i)$, and the discriminator focuses on discriminating the real and generated samples. The classifier with ABN is trained by ARPL and FT. An adversarial mechanism between the known classes and reciprocal points is introduced here. On the one hand, the generated images let the discriminator identify positive samples (i.e., those close to known samples). On the other hand, the generated images should be samples unknown to the classifier, so the embedding feature of the neural network with respect to these samples is close to all reciprocal points.}
		\label{fig:GAN}
	\end{figure}
	
	\subsection{Learning the Confused Generator}
	
	Here, a new training strategy is proposed to learn a confused generator.
	Unlike the common GAN \cite{goodfellow2014generative}, we want to employ the generator to recover some confusing samples from $\mathcal{O}_G$ instead of known samples from $\mathcal{S}_k$. 
	As shown in Fig.~\ref{fig:GAN}, the proposed instantiated adversarial enhancement framework contains three main components: the discriminator $D$, the generator $G$, and the classifier $C$ with a deep embedding function $\mathcal{C}$. 
	The classifier with ARPL represents the probability that the sample belongs to each known category.
	The generator maps a latent variable $z$ from a prior distribution $P_{pri}(z)$ to the generated outputs $G(z)$, and the discriminator $D: \mathcal{X} \to \left[ 0, 1 \right]$ represents the probability of sample $x$ being from the real distribution or a fake distribution. 
	Then, given $ \{ z_{1}, \cdots, z_n \} $ from the prior $P_{pri} (z)$ and the known samples $ \{ x_1, \cdots, x_n \}$, the discriminator is optimized to discriminate the real and generated samples:
	\begin{equation}
		\max_{D}{\frac{1}{n}\sum^n_{i=1}[\log{D(x_i)}+\log(1-D(G(z_i)))]}.
		\label{Eqn:discriminator}
	\end{equation}
	In contrast, the generator expects the generated samples to be closer to the known classes so as to deceive the discriminator:
	\begin{equation}
		\max_{G}{\frac{1}{n}\sum^n_{i=1}[\log{D(G(z_i))}]}.
		\label{Eqn:generator1}
	\end{equation}
	
	In order to confuse the generator, an adversarial mechanism between the known classes and reciprocal points is introduced here. 
	It encourages the generator to create samples close to each center $\mathcal{P}^k$ of the open space $\mathcal{O}_k$.
	Similar to Eq. \eqref{Eqn:OG}, it is equivalent to encouraging the generated images to be close to the global open space $\mathcal{O}_G$. 
	Formally, the generator is optimized through the classifier:
	\begin{equation}
		\max_{G} \ \frac{1}{n}\sum^n_{i=1}[-\frac{1}{N}\sum^{N}_{k=1} S(z_i, \mathcal{P}^k) \cdot log(S(z_i, \mathcal{P}^k))],
		\label{Eqn:generator2}
	\end{equation}
	where $S(z_i, \mathcal{P}^k) = softmax(d_e(\mathcal{C}(G(z_i)), \mathcal{P}^k))$. The maximum value of Eq. \eqref{Eqn:generator2} for confusing samples is achieved when the embedding of these  samples are close to all reciprocal points. A theorem is introduced to better illustrate this.
	
	\begin{lemma}
		\label{lemma:entropy}
		For a neural network whose logit layer is based on reciprocal points and $x = G(z)$, Eq. \eqref{Eqn:generator2} is maximized when the distances between the deep feature vector $\mathcal{C}(x)$ and all reciprocal points are equal: $\forall n \in N: S(z_i, \mathcal{P}^k) = \frac{1}{N}$ and the entropy of the distance distribution is maximized.
	\end{lemma}
	
	\noindent
	For $x = G(z)$, Eq. \eqref{Eqn:generator2} is the same in form as the information entropy over the per-class softmax scores. Thus, based on Shannon entropy \cite{shannon1948mathematical}, it is intuitive that Eq. \eqref{Eqn:generator2} is maximized when all values are equal. 
	
	By combining these two mechanisms for confrontation, the generator is optimized by:
	\begin{equation}
		\begin{split}
			\max_{G} \ \frac{1}{n}\sum^n_{i=1}[\log{D(G(z_i))} + \beta \cdot H(z_i, \mathcal{P})],
			\label{Eqn:generator}
		\end{split}
	\end{equation}
	where $\beta$ is a hyperparameter for controlling the weight of the information entropy loss, and $H(z_i, \mathcal{P}) = -\frac{1}{N}\sum^{N}_{k=1} S(z_i, \mathcal{P}^k) \cdot log(S(z_i, \mathcal{P}^k)) $ is the information entropy function. 
	As in the framework illustrated in Fig.~\ref{fig:GAN},  we aim to generate samples more similar with known samples; this also forces the generator to create samples that balance the distance for all the reciprocal points, so that they are close to the global open space $\mathcal{O}_G$.
	If the generated samples are far from the boundary of the known samples, the loss in Eq. \eqref{Eqn:generator1} should be large. 
	To deceive the discriminator, the generator generates samples similar to the known classes, which also makes the features of the generated samples close to the known classes and far from some reciprocal points. 
	Hence, the loss in Eq. \eqref{Eqn:generator2} should be large. 
	Therefore, one expects that the proposed loss will encourage the generator to produce samples that are on the boundary of the global open space, as shown in Fig.~\ref{fig:osg_ARPL} and Fig.~\ref{fig:osg_examples}.
	
	\begin{figure}[!tb]
		\centering
		\subfigure[Retrieval Examples]{
			\begin{minipage}[t]{0.48\linewidth}
				\centering
				\includegraphics[width=\linewidth]{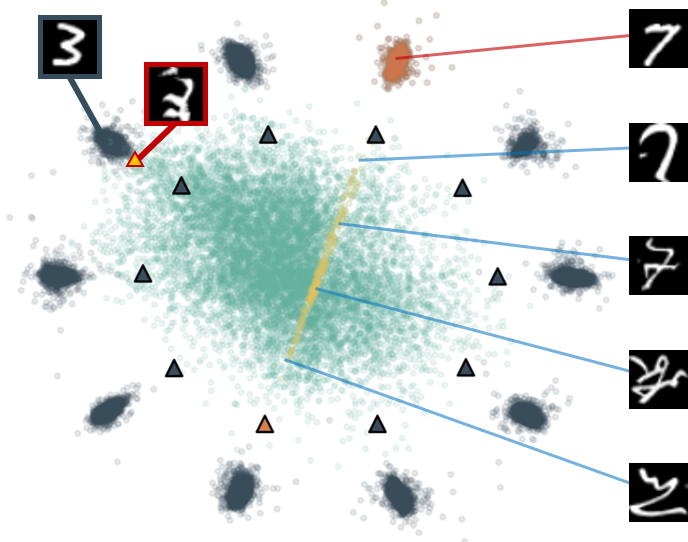}
				\label{fig:retrieval}
			\end{minipage}%
		}
		\subfigure[Generated Samples]{
			\begin{minipage}[t]{0.48\linewidth}
				\centering
				\includegraphics[width=\linewidth]{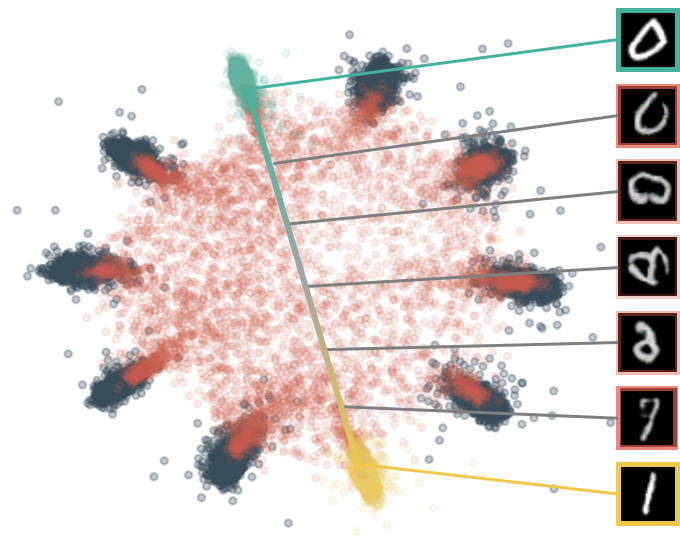}
				\label{fig:osg_examples}
			\end{minipage}%
		}
		\caption{(a) Several retrieval examples from KMNIST. The orange circle represents the known class number $7$, and the orange triangle represents the reciprocal point corresponding to number $7$. The red sample in the upper left corner indicates a failure case, which is very similar to the known class number $3$. (b) Generated confusing images from adversarial training with the ARPL classifier on MNIST. The topmost and bottommost correspond to the known training images.} 
		
	\end{figure}
	
	\subsection{Reliability Enhancement}
	Consider the generated samples as unknown data $\mathcal{D}_U$, and consider the ultimate goal of training a better feature space, where the open space is minimized.
	Therefore, the classifier $C$ is optimized by the generated confusing samples as:
	\begin{equation}
		\min_{C} \ \frac{1}{n}\sum^n_{i=1}[\mathcal{L}(x_i, y_i) - \beta \cdot H(z_i, \mathcal{P})],
		\label{Eqn:classifier}
	\end{equation}
	where $\mathcal{L}$ is the overall loss of ARPL.
	These generated samples are used to estimate the unknown distribution of $\mathcal{O}_G$, to reduce the open space risk by reducing the size of $\mathcal{O}_G$ (as shown in Fig.~\ref{fig:osg_ARPL_OS}).
	
	Note that the known samples and generated samples are processed independently in Eq. \eqref{Eqn:classifier}. 
	In this circumstance, the generated samples could confuse the classifier because of their different distributions with known samples, resulting in inaccurate statistics.
	To disentangle this mixed distribution into two underlying distributions for the known and confusing samples, we propose \emph{Auxiliary Batch Normalization} (ABN) to guarantee that the normalization statistics obtained exclusively for the confusing examples.
	Specifically, Batch Normalization \cite{ioffe2015batch} normalizes the input features by the mean and variance computed within each minibatch, where the input features should come from a single or similar distribution \cite{xie2020adversarial}. 
	As illustrated in Fig.~\ref{fig:GAN}, ABN helps to disentangle mixed distributions by keeping separate BNs for features that belong to different domains.
	Compared with the two-component mixture distribution (known and confusing samples), this auxiliary BN is able to effectively block the negative impact of confusing samples on known class discrimination.
	The ablation studies in Sec.~\ref{EXP:OOD} demonstrate that such disentangled learning with multiple BNs can improve the performance. 
	
	
	
	Finally, the discriminator and the classifier can be used to improve each other through the confused generator. 
	This naturally suggests a joint training scheme where the classifier improves the generator and vice versa.
	In the same way, it is inevitably valid for the discriminator.
	An alternating algorithm is designed to optimize the above objective efficiently, as shown in Alg. \ref{alg:2}. 
	After training each classifier with the confusing samples (step 6 in Alg. \ref{alg:2}), we add \emph{Focus Training} (FT) and use the known class to train the classifier again.
	The purpose of this scheme is to encourage the classifier to focus on the known classification and correct the deviation of giving too much attention to the confusing samples.
	
	Compared with \cite{lee2017training, dai2017good, neal2018open, ge2017generative}, the main differences of the proposed instantiated adversarial enhancement are as follows: First, the proposed method uses the adversarial mechanism between the close space $\mathcal{S}_k$ and the global open space $\mathcal{O}_G$ formed by the reciprocal points. 
	Second, the generated confusing samples and known samples are validated as two different distributions to accurately estimate the statistics of known and unknown classes. In addition, the images generated by our method cover the whole low response of the unknown feature space (as shown in Fig.~\ref{fig:osg}), and contain a certain quantity of confusing images similar to the known classes (as illustrated in Fig.~\ref{fig:osg_examples}).

	\begin{algorithm}[!tb]
		\fontsize{8}{8}
		\renewcommand{\algorithmicrequire}{\textbf{Input:}}
		\renewcommand{\algorithmicensure}{\textbf{Output:}}
		\caption{The instantiated adversarial enhancement algorithm.}
		\label{alg:2}
		\begin{algorithmic}[1]
			\REQUIRE Training data $\{ x_i \}$. Initialized parameters $\theta_D$ of the discriminator $D$, $\theta_G$ of the confusing generator $G$ and $\theta_C$ of the classifier $C$ with $\mathcal{P}$ and $R$ in loss layers, respectively. Hyperparameter $ \lambda, \gamma, \beta $.
			\ENSURE The parameters $\theta_D$, $\theta_G$, $\theta_C$, $\mathcal{P}$ and $R$.
			\REPEAT
			\STATE Sample $ \{ z_{1}, ..., z_N \} $ from prior $P_{pri} (z)$ and known samples $ \{ (x_1, y_1), ..., (x_N, y_N) \}$.
			\STATE Update the discriminator parameters $\theta_D$ by ascending its stochastic gradient: \\
			\begin{equation}
				\nonumber
				\nabla_{\theta_D}\frac{1}{n}\sum^n_{i=1}[\log{D(x_i)}+\log(1-D(G(z_i)))].
			\end{equation}
			\STATE Update the generator parameters $\theta_G$ by ascending its stochastic gradient: \\
			\begin{equation}
				\nonumber
				\nabla_{\theta_G}\frac{1}{N}\sum^N_{i=1}[\log{D(G(z_i))} + \beta \cdot H(z_i \mathcal{P})].
			\end{equation}
			\STATE Update the classifier parameters $\theta_C$ with $\mathcal{P}$ and $R$ by by descending its stochastic gradient: 
			\begin{equation}
				\nonumber
				\nabla_{\theta_C}\frac{1}{n}\sum^n_{i=1}[L(x_i, y_i) - \beta \cdot H(z_i, \mathcal{P})].
			\end{equation}
			\STATE Update the classifier parameters $\theta_C$ with $\mathcal{P}$ and $R$ by minimizing $\frac{1}{n}\sum^n_{i=1}L(x_i, y_i)$.
			\UNTIL{\textit{convergence}}
		\end{algorithmic}
	\end{algorithm}

	\subsection{Unknown Classes Detection}
	
	Based on Eq. \eqref{Eqn:r-point}, the unknown samples naturally have a closer distance to all reciprocal points than the samples of known classes. 
	Therefore, the probability that the test instance $x$ belongs to one of the known classes is proportional to the distance between $x$ and the farthest reciprocal point corresponding to category $k$:
	\begin{equation}
		p(known|x) \propto \max_{k \in \lbrace 1, \dots, N \rbrace}{d(f(x), \mathcal{P}^k)}.
	\end{equation}
	One of the key issues in OSR models is \emph{what's a good score for open set recognition?} (i.e., identifying a class as known or unknown).
	Since how rare or common samples from unknown space are not known in the actual scenario, the approaches to OSR that require an arbitrary threshold or sensitivity for comparison are unreasonable \cite{neal2018open}.
	Thus, the difference between known and unknown probability belonging to any known classes is used to measure the learned models' ability to detect unknown samples, which provides a calibration-free measure of the detection performance.

	\begin{table*}[!tb]
		\caption{The AUROC results of on detecting known and unknown samples. Results are averaged among five randomized trials.}
		\centering
		\label{tab:osdi}
		\begin{tabular}{cllllll}
			\specialrule{.16em}{0pt} {.65ex}
			Method & MNIST & SVHN & CIFAR10 & CIFAR+10 & CIFAR+50 & TinyImageNet \\
			\midrule
			Softmax & 97.8 $\pm$ 0.2 & 88.6 $\pm$ 0.6 & 67.7 $\pm$ 3.2 & 81.6 $\pm$ - & 80.5 $\pm$ - & 57.7 $\pm$ - \\
			Openmax \cite{bendale2016towards}  & 98.1 $\pm$ 0.2 & 89.4 $\pm$ 0.8 & 69.5 $\pm$ 3.2 & 81.7 $\pm$ - & 79.6 $\pm$ - & 57.6 $\pm$ - \\
			G-OpenMax \cite{ge2017generative}  & 98.4 $\pm$ 0.1 & 89.6 $\pm$ 0.6 & 67.5 $\pm$ 3.5 & 82.7 $\pm$ - & 81.9 $\pm$ - & 58.0 $\pm$ - \\
			OSRCI \cite{neal2018open}  & 98.8 $\pm$ 0.1 & 91.0 $\pm$ 0.6 & 69.9 $\pm$ 2.9 & 83.8 $\pm$ - & 82.7 $\pm$ - & 58.6 $\pm$ - \\
			CROSR \cite{yoshihashi2019classification} & 99.1 $\pm$ - & 89.9 $\pm$ - & 88.3 $\pm$ - & 91.2 $\pm$ - & 90.5 $\pm$ - & 58.9 $\pm$ - \\
			C2AE \cite{oza2019c2ae}  & 98.9 $\pm$ 0.2 & 92.2 $\pm$ 0.9 & 89.5 $\pm$ 0.8 & 95.5 $\pm$ 0.6 & 93.7 $\pm$ 0.4 & 74.8 $\pm$ 0.5 \\
			\midrule
			RPL \cite{chen_2020_ECCV}  & 98.9 $\pm$ 0.1 & 93.4 $\pm$ 0.5 & 82.7 $\pm$ 1.4 & 84.2 $\pm$ 1.0 & 83.2 $\pm$ 0.7 & 68.8 $\pm$ 1.4 \\
			ARPL & 99.6 $\pm$ 0.1 & 96.3 $\pm$ 0.3 & 90.1 $\pm$ 0.5 & 96.5 $\pm$ 0.6 & 94.3 $\pm$ 0.4 & 76.2 $\pm$ 0.5 \\
			ARPL+CS & \textbf{99.7} $\pm$ 0.1 & \textbf{96.7} $\pm$ 0.2 & \textbf{91.0} $\pm$ 0.7 & \textbf{97.1} $\pm$ 0.3 & \textbf{95.1} $\pm$ 0.2 & \textbf{78.2} $\pm$ 1.3 \\
			\specialrule{.16em}{.4ex}{0pt}
		\end{tabular}
	\end{table*}
	
	\begin{table*}[!tb]
		\caption{The open set classification rate (OSCR) curve results of open set recognition. Results are averaged among five randomized trials.}
		\centering
		\label{tab:osr}
		\begin{tabular}{cllllll}
			\specialrule{.16em}{0pt} {.65ex}
			Method & MNIST & SVHN & CIFAR10 & CIFAR+10 & CIFAR+50 & TinyImageNet \\
			\midrule
			Softmax & 99.2 $\pm$ 0.1 & 92.8 $\pm$ 0.4 & 83.8 $\pm$ 1.5 & 90.9 $\pm$ 1.3 & 88.5 $\pm$ 0.7 & 60.8 $\pm$ 5.1 \\
			GCPL \cite{yang2018robust} & 99.1 $\pm$ 0.2 & 93.4 $\pm$ 0.6 & 84.3 $\pm$ 1.7 & 91.0 $\pm$ 1.7 & 88.3 $\pm$ 1.1  & 59.3 $\pm$ 5.3 \\
			RPL \cite{chen_2020_ECCV}  & 99.4 $\pm$ 0.1 & 93.6 $\pm$ 0.5 & 85.2 $\pm$ 1.4 & 91.8 $\pm$ 1.2 & 89.6 $\pm$ 0.9 & 53.2 $\pm$ 4.6 \\
			ARPL & 99.4 $\pm$ 0.1 & 94.0 $\pm$ 0.6 & 86.6 $\pm$ 1.4 & 93.5 $\pm$ 0.8 & 91.6 $\pm$ 0.4 & 62.3 $\pm$ 3.3 \\
			ARPL+CS & \textbf{99.5} $\pm$ 0.1 & \textbf{94.3} $\pm$ 0.3 & \textbf{87.9} $\pm$ 1.5 & \textbf{94.7} $\pm$ 0.7 & \textbf{92.9} $\pm$ 0.3 & \textbf{65.9} $\pm$ 3.8 \\
			\specialrule{.16em}{.4ex}{0pt}
		\end{tabular}
	\end{table*}
	
	\section{Experiments}
	
	\subsection{Implementation Details}
	\label{implementation}
	$\gamma$ is set as 1.0, and $ \lambda $ and $\beta$ are set to 0.1, in all training phases.
	They all are determined by cross validation.
	The reciprocal points are initialized by a random normal distribution and each margin is initialized to one.
	For open set recognition, a global average pooling is added after the final convolution layer of the encoder.
	In experiments for out-of-distribution samples, the output after global average pooling (GAP) of the ResNet is utilized as the feature.
	The dimension of the reciprocal point is consistent with the output of the GAP.
	Each known class is assigned one reciprocal point for training.
	In addition to MNIST, random center-cropping and random horizontal flips are used as data augmentation.
	The images in TinyImageNet are resized to 64x64 in the experiments. 
	During testing, only the BN for known classes is used to generate the deep feature for known and unknown classes.
	
	\subsection{Experiments for Open Set Recognition}
	\label{EXP:OSR}
	
	\textbf{Datasets.} Similar to \cite{oza2019c2ae}, a simple summary of these protocols for each dataset is provided: 
	\begin{itemize}
		\item \textbf{MNIST}, \textbf{SVHN}, \textbf{CIFAR10}. For MNIST \cite{lecun1998gradient}, SVHN \cite{netzer2011reading} and CIFAR10 \cite{krizhevsky2009learning}, 6 known classes and 4 unknown classes are randomly sampled. 
		\item \textbf{CIFAR+10}, \textbf{CIFAR+50}. For the CIFAR+$N$ experiments, 4 classes are sampled from CIFAR10 for training. $N$ nonoverlapping classes are used as unknown classes, which are sampled from the CIFAR100 dataset \cite{krizhevsky2009learning}. 
		\item \textbf{TinyImageNet}. For experiments with TinyImageNet \cite{russakovsky2015imagenet}, 20 known classes and 180 unknown classes are randomly sampled for evaluation.
	\end{itemize}

	\noindent
	\textbf{Evaluation Metrics.} 
	Since how rare or common the samples from unknown space are is not known in the actual scenario, the approaches to OSR that require an arbitrary threshold or sensitivity for comparison are unreasonable \cite{neal2018open}. 
	A threshold-independent metric, the Area Under the Receiver Operating Characteristic (AUROC) curve \cite{neal2018open}, is taken as one of the evaluation metrics.
	The AUROC curve is threshold-independent metric \cite{davis2006relationship} that plots the true positive rate against the false positive rate by varying a threshold \cite{lee2017training}. 
	It can be interpreted as the probability that a positive example is assigned a higher detection score than a negative example \cite{fawcett2006introduction}. 
	
	However, the AUROC evaluates only the distinction between known and unknown classes, and does not consider the accuracy of known classes in open set recognition, which has been however hidden by this gold-standard "fair" metric \cite{dhamija2018reducing,fang2021learning}.
	To adapt it to the case of open set recognition, we introduce \emph{Open Set Classification Rate} (OSCR) \cite{dhamija2018reducing} as a new evaluation metric. 
	Let $\delta$ be a score threshold. 
	The \emph{Correct Classification Rate} (CCR) is the fraction of samples where the correct class $k$ has maximum probability and has a probability greater than $\delta$: 
	\begin{equation} 
		CCR(\delta)=\frac{|\{x \in \mathcal{D}_\mathcal{T}^k \land arg max_k P(k|x) = \hat{k} \land P(\hat{k}|x) \ge \delta \} }{|\mathcal{D}_\mathcal{T}^k|}.
	\end{equation}
	The \emph{False Positive Rate} (FPR) is the fraction of samples from unknown data $\mathcal{D}_U$ that are classified as \emph{any} known class $k$ with a probability greater than $\delta$:
	\begin{equation}
		FPR(\delta) = \frac{|\{x | x \in \mathcal{D}_U \land max_k{P(k|x) \ge \delta}\}|}{|\mathcal{D}_U|}.
	\end{equation}
	A larger value of the OSCR indicates better detection performance. 
	Following the protocol in \cite{neal2018open}, the AUROC and the OSCR are averaged over five randomized trials.
	
	\noindent
	\textbf{Network Architecture.} 
	The classifier for this experiment is the same as the neural network used in \cite{neal2018open}. 
	Apart from the Adam optimizer \cite{kingma2014adam} used in TinyImageNet, all classifiers are trained with the momentum stochastic gradient descent (Momentum SGD) optimizer \cite{qian1999momentum}.
	The learning rate of the classifier starts from 0.1 and decreases by a factor of 0.1 every 30 epochs in the training progress.
	The confused generator and the discriminator are the same as thos in \cite{lee2017training}, and are trained by the Adam optimizer \cite{kingma2014adam} with a learning rate of 0.0002. 
	More details are given in Section \ref{implementation}.
	
	\noindent
	\textbf{Result Comparison.}
	As shown in Table~\ref{tab:osdi}, ARPL using only known training samples significantly outperforms most other approaches (including traditional discriminative method-based neural networks and some complicated generative methods \cite{neal2018open, oza2019c2ae, yoshihashi2019classification} for OSR) significantly.
	These generative methods \cite{neal2018open, oza2019c2ae, yoshihashi2019classification} consider using decoder to optimize the deep feature space, but they do not pay attention to the characteristics of the unknown distribution in deep feature space.
	Instead, ARPL pushes the known classes away from the unknown classes through reciprocal points to form a better discriminative feature space.
	Furthermore, ARPL with Confusing Samples (ARPL+CS) performs significantly better than other recent state-of-the-art generative methods \cite{neal2018open, oza2019c2ae, yoshihashi2019classification, sun2020conditional} and ARPL, especially on SVHN, CIFAR, and TinyImageNet.
	This further demonstrates the superiority of the proposed method, and these confusing samples can effectively improve the reliability of the neural network with ARPL.
	
	
	
	
	Moreover, we design a new OSR experiment for a more reasonable comparison. 
	First, we abandon the baseline with hinge loss in \cite{neal2018open}, which can lead to some optimization difficulties. 
	The more robust cross-entropy loss is used as a new baseline in this experiment.
	Second, we introduce a new evaluation metric, OSCR\cite{dhamija2018reducing}, to comprehensively evaluate the classification performance for known and unknown class detection under different thresholds.
	Finally, under the same five known and unknown splits, we report the performance of the average for five trials.
	
	Compared with the experiment based on the AUROC in Table~\ref{tab:osdi}, most tasks in this new OSR experiment become more difficult because these methods should balance the unknown detection with classification for known classes.
	We compared four discriminative methods based on a neural network as shown in Table~\ref{tab:osr}. 
	ARPL shows excellent performance compared to cross-entropy loss, GCPL \cite{yang2018robust}, and RPL \cite{chen_2020_ECCV}.
	Moreover, with the assistance of confusing samples, the OSCR of ARPL was greatly improved.
	In particular, on TinyImageNet, the performance was improved 3.6\% compared with that of ARPL.
	These results show that ARPL and ARPL+CS are able to effectively improve the detection ability of unknown samples while ensuring the accuracy of known class classification.

	\begin{table*}[!tb]
		\caption{Distinguishing in- and out-of-distribution test set data for image classification under various validation setups. All values are percentages and the best results are indicated in bold.}
		\centering
		\small
		\label{tab:ood}
		\begin{tabular}{ccccccccccc}
			\specialrule{.16em}{0pt} {.65ex}
			\multirow{2}*{Method} & \multicolumn{5}{c}{In: \textbf{CIFAR10} / Out: \textbf{CIFAR100}} & \multicolumn{5}{c}{In: \textbf{CIFAR10} / Out: \textbf{SVHN}} \\
			\cmidrule(r){2-6} \cmidrule(r){7-11}
			& TNR & AUROC & DTACC & AUIN & AUOUT & TNR & AUROC & DTACC & AUIN & AUOUT \\
			\specialrule{.1em}{0pt} {.65ex}
			Cross Entropy & 31.9 & 86.3 & 79.8 & 88.4 & 82.5 & 32.1 & 90.6 & 86.4 & 88.3 & 93.6 \\
			GCPL\cite{yang2018robust} & 35.7 & 86.4 & 80.2 & 86.6 & 84.1 & 41.4 & 91.3 & 86.1 & 86.6 & 94.8 \\
			RPL\cite{chen_2020_ECCV} & 32.6 & 87.1 & 80.6 & 88.8 & 83.8 & 41.9 & 92.0 & 87.1 & 89.6 & 95.1 \\
			ARPL & \textbf{47.0} & \textbf{89.7} & \textbf{82.6} & \textbf{90.5} & \textbf{87.8} & \textbf{53.8} & \textbf{93.2} & \textbf{87.2} & \textbf{90.3} & \textbf{95.8} \\
			\midrule
			JCL \cite{lee2017training} & 35.8 & 87.8 & 81.2 & 89.1 & 88.6 & 34.7 & 92.2 & 88.5 & 90.6 & 92.9 \\
			ARPL+CS (w/o ABN) & 46.3 & 88.8 & 81.7 & 89.1 & 87.1 & 44.7 & 90.9 & 84.1 & 86.2 & 94.6 \\
			ARPL+CS (w/o FT) & 46.7 & 89.7 & 82.4 & 90.7 & 87.8 & 71.0 & 95.6 & 91.1 & 94.0 & 96.8 \\
			ARPL+CS & \textbf{48.5} & \textbf{90.3} & \textbf{83.4} & \textbf{91.1} & \textbf{88.4} & \textbf{79.1} & \textbf{96.6} & \textbf{91.6} & \textbf{94.8} & \textbf{98.0} \\
			\specialrule{.16em}{.4ex}{0pt}
		\end{tabular}
	\end{table*}
	
	\subsection{Experiments for Out-of-Distribution Detection}
	\label{EXP:OOD}
	
	\noindent
	\textbf{Datasets.}
	we adopt three image datasets that represent the most challenging pairs of common OOD detection benchmarks \cite{hendrycks17baseline}, CIFAR10, CIFAR100 \cite{krizhevsky2009learning}, SVHN \cite{netzer2011reading} for evaluation. 
	CIFAR100 and SVHN are the near OOD dataset and far OOD dataset for CIFAR10, respectively. 
	Note that the CIFAR10 and CIFAR100 classes are mutually exclusive. 
	
	\noindent
	\textbf{Evaluation Metrics.}
	Referring to the evaluation index in \cite{hendrycks17baseline,hendrycks2018deep,lee2017training,liang2017enhancing}, the AUROC, the true negative rate (TNR) at $95\%$ true positive rate (TPR), the area under the precision-recall curve (AUPR), and the detection accuracyare are adopted for evaluation:
	\begin{itemize}
		\item \textbf{True negative rate (TNR) at $95\%$ true positive rate (TPR).} Let TP, TN, FP, and FN denote true positives, true negatives, false positives and false negatives, respectively. We measure $TNR=TN/(TP+TN)$, when $TPR=TP/(FP+FN)$ is  $95\%$ 
		\item \textbf{Area under the precision-recall curve (AUPR).} The PR curve is graph plotting precision $=TP/(TP+FP)$ against recall $=TP/(TP+FN)$ by varying a threshold. The \textbf{AUIN} (or \textbf{AUOUT}) is the AUPR where the in- (or out-of-) distribution samples are specified as positive.
		\item \textbf{Detection accuracy (DTACC).} This metric corresponds to the maximum classification probability over all possible thresholds $\delta$. We assume that both positive and negative examples have equal probability of appearing in the test set, i.e., $P(x \in P_{in})=P(x \in P_{out})=0.5$
	\end{itemize}
	
	\noindent
	\textbf{Network Architecture.}
	We demonstrate the effectiveness of the proposed method using ResNet with 34 layers \cite{he2016deep} on various vision datasets.
	All classifiers are trained for 100 epochs with batch size 128 with the Adam optimizer \cite{kingma2014adam}.
	The learning rate of the classifier starts at 0.1 and decreases by a factor of 0.1 every 30 epochs in the training process.
	The confused generator and the discriminator are the same as those in \cite{lee2017training}, and are trained by the Adam optimizer \cite{kingma2014adam} with the learning rate of 0.0002. 
	More details are given in Section \ref{implementation}.
	
	\noindent
	\textbf{Result Comparison.}
	As shown in Table~\ref{tab:ood}, the ARPL outperforms training by three methods for in-distribution data only; the methods are cross-entropy loss (Baseline), GCPL \cite{yang2018robust} and RPL \cite{chen_2020_ECCV}.
	For CIFAR100 containing both test samples that are near as well as far OOD, ARPL is more than 3\% better than the RPL in terms of AUROC.
	This confirms that ARPL can differentiate OOD classes that are near or far from in-distribution data.
	Additionaly, our methods are compared with a generative OOD model, named Joint Confidence Loss (JCL) \cite{lee2017training}, which adopts a similar mechanism with our instantiated adversarial enhancement.
	Due to considering the difference between the unknown and known samples in deep feature space, the performance of ARPL without the auxiliary training of confusing samples even is better than that of JCL.
	
	Here we analyze the role of ABN and FT for instantiated adversarial enhancement.
	The performance of ARPL+CS w/o ABN is much lower than that of ARPL. 
	This is consistent with our assumption that confusing samples and known images have different underlying distributions. 
	Although the generator aims to generate images that are consistent with the distribution of known classes, the distributions of the generated images and known classes gradually separate after adding our adversarial mechanism. 
	One BN for mixed distributions would influence the performance to detect OOD samples. 
	After adding ABN (ARPL+CS w/o FT), the performance is gradually improved, especially for far OOD detection.  
	
	However, the detection of near OOD is not improved compared with that of ARPL. 
	It could be that confusing samples affect the discriminative feature of the known class in the training stage. 
	Based on the prior initialization of confusing samples, we use FT to encourage the classifier to pay more attention to the classification of known classes. 
	The performance of both near and far OOD was further improved as shown in Table~\ref{tab:ood}.
	
	\begin{figure*}[!tb]
		\centering
		\subfigure[RPL($\lambda = 0$)]{
			\begin{minipage}[t]{0.24\linewidth}
				\centering
				\includegraphics[width=\linewidth]{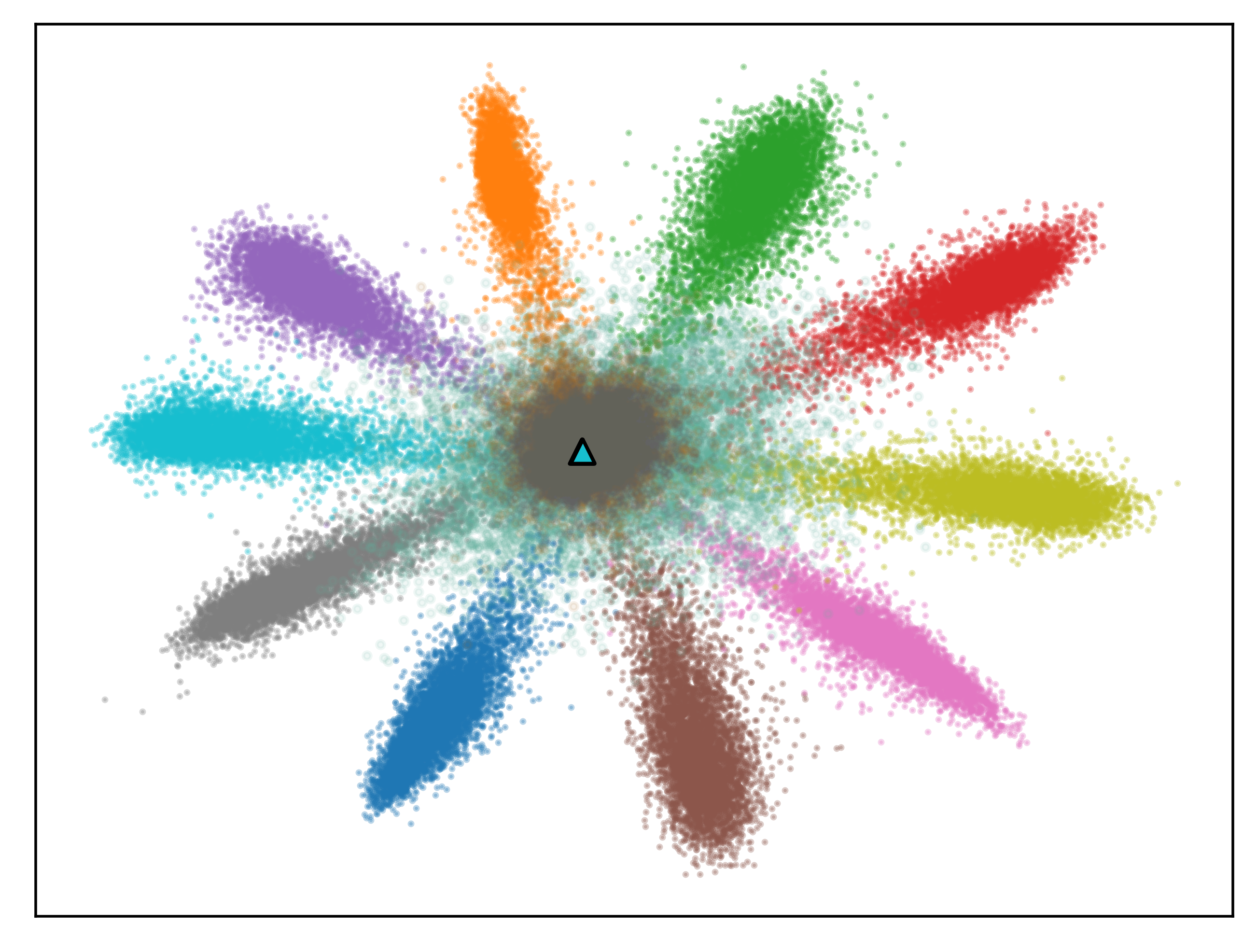}
				\includegraphics[width=\linewidth]{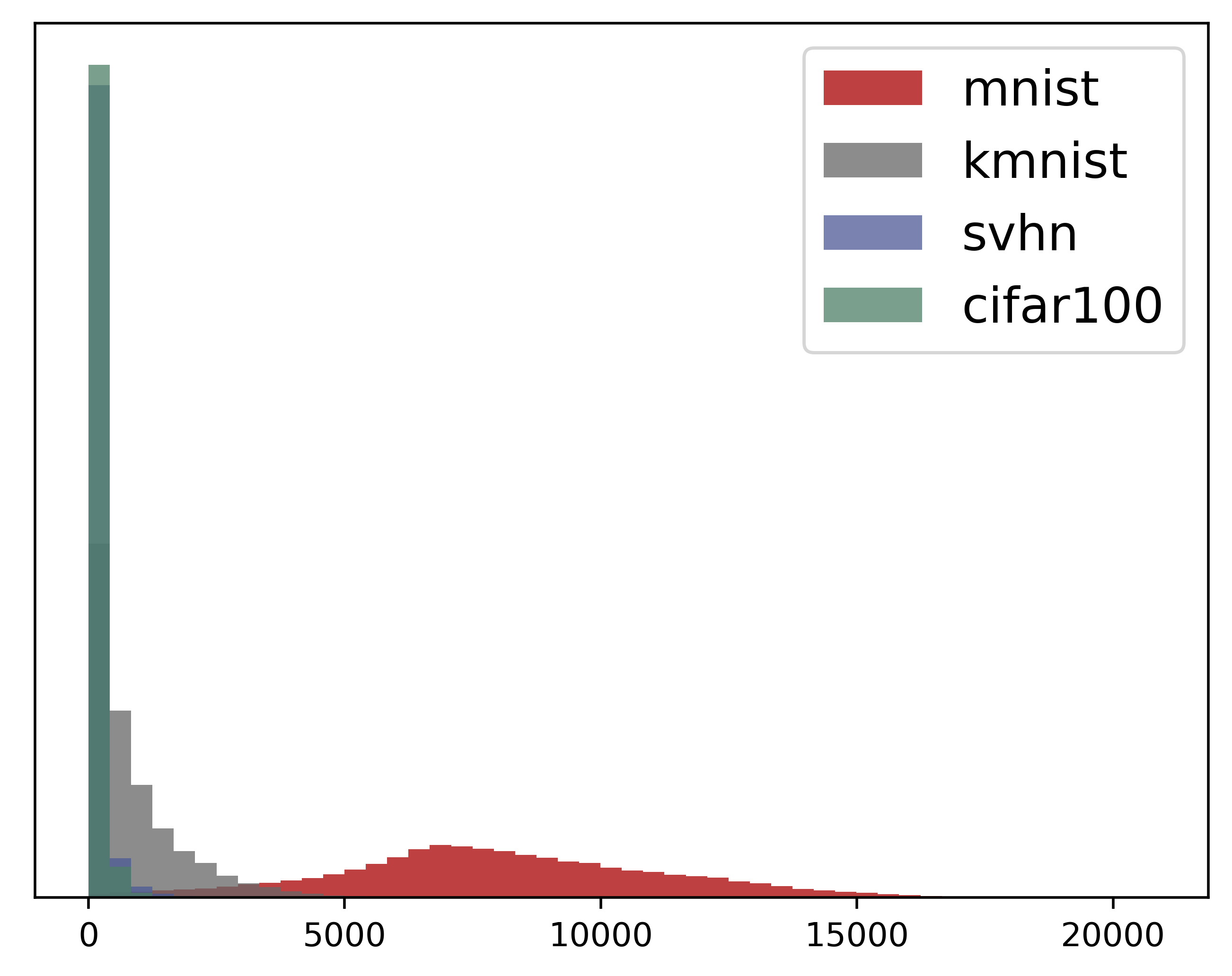}
				\label{fig:RPL}
			\end{minipage}%
		}
		\subfigure[ARPL($\lambda = 0$)]{
			\begin{minipage}[t]{0.24\linewidth}
				\centering
				\includegraphics[width=\linewidth]{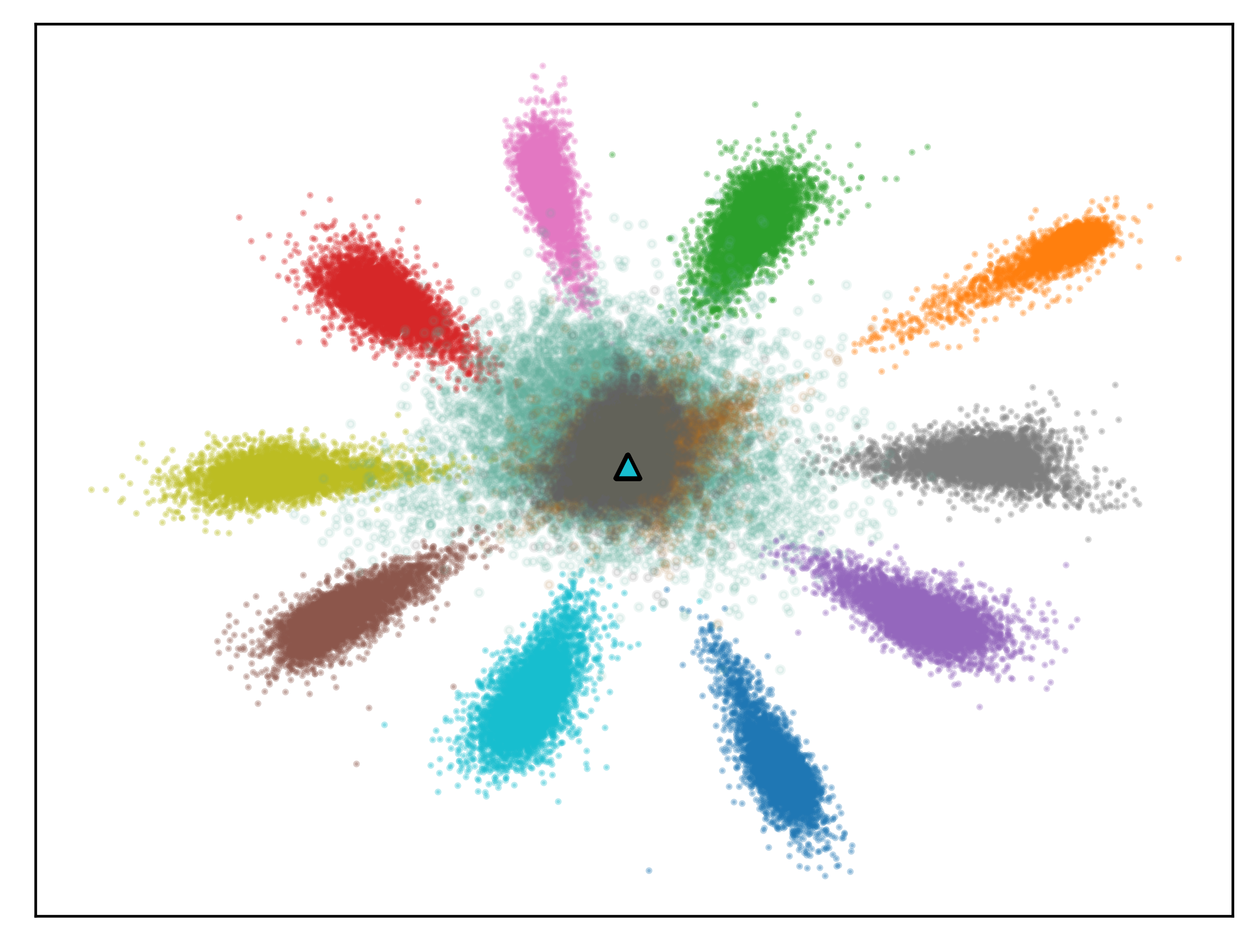}
				\includegraphics[width=\linewidth]{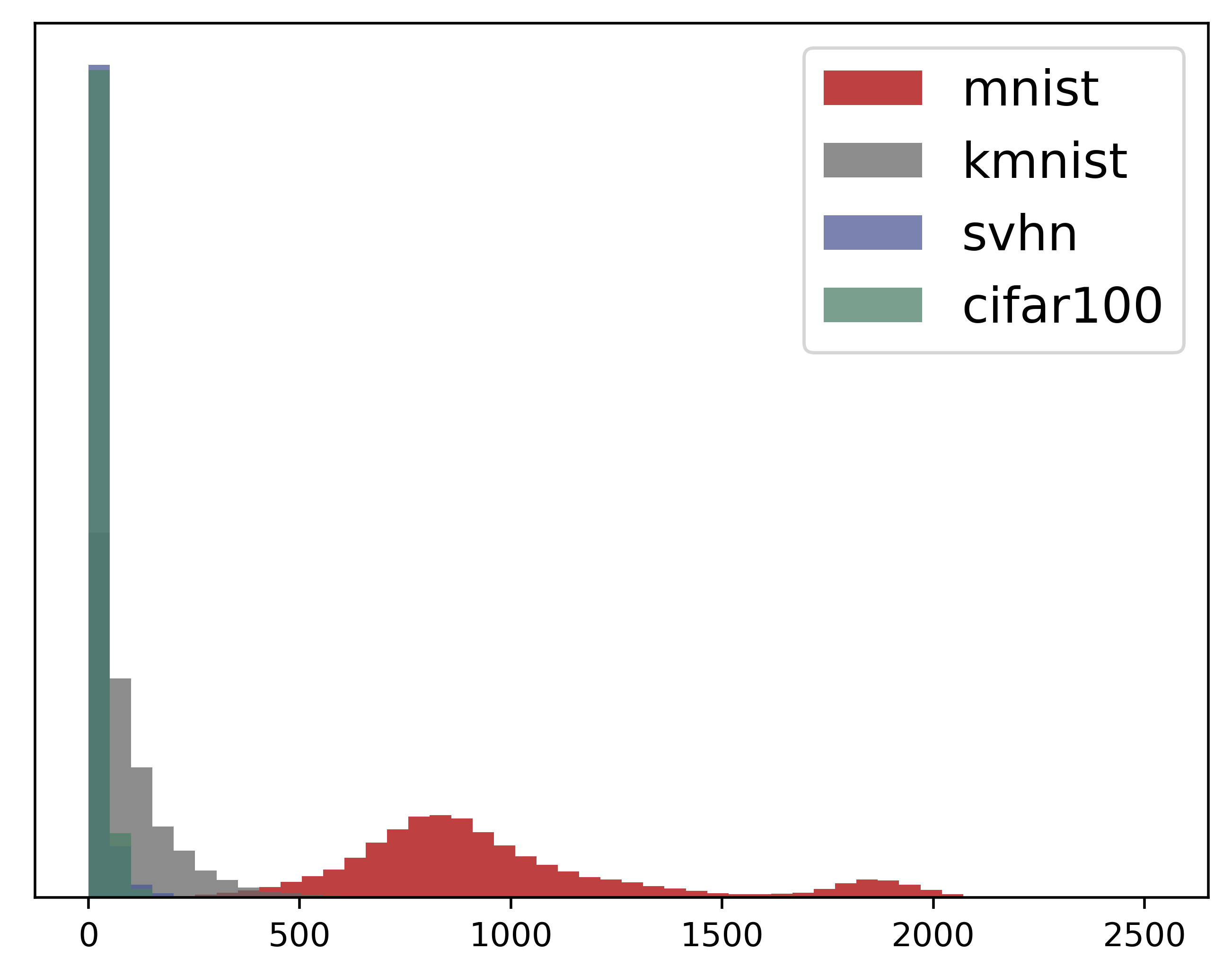}
				\label{fig:ARPL_0}
			\end{minipage}%
		}%
		\subfigure[ARPL($\lambda = 0.1$)]{
			\begin{minipage}[t]{0.24\linewidth}
				\centering
				\includegraphics[width=\linewidth]{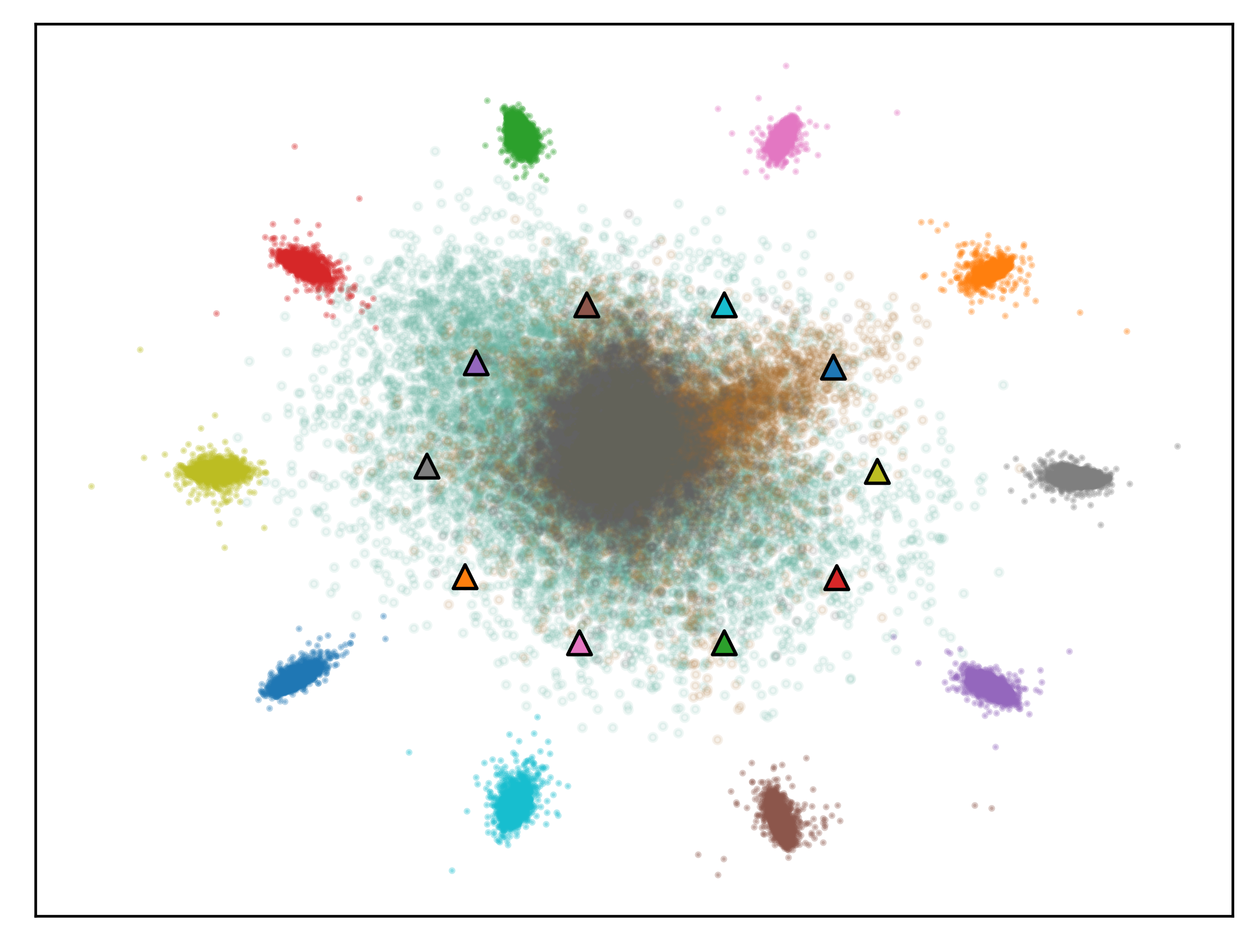}
				\includegraphics[width=\linewidth]{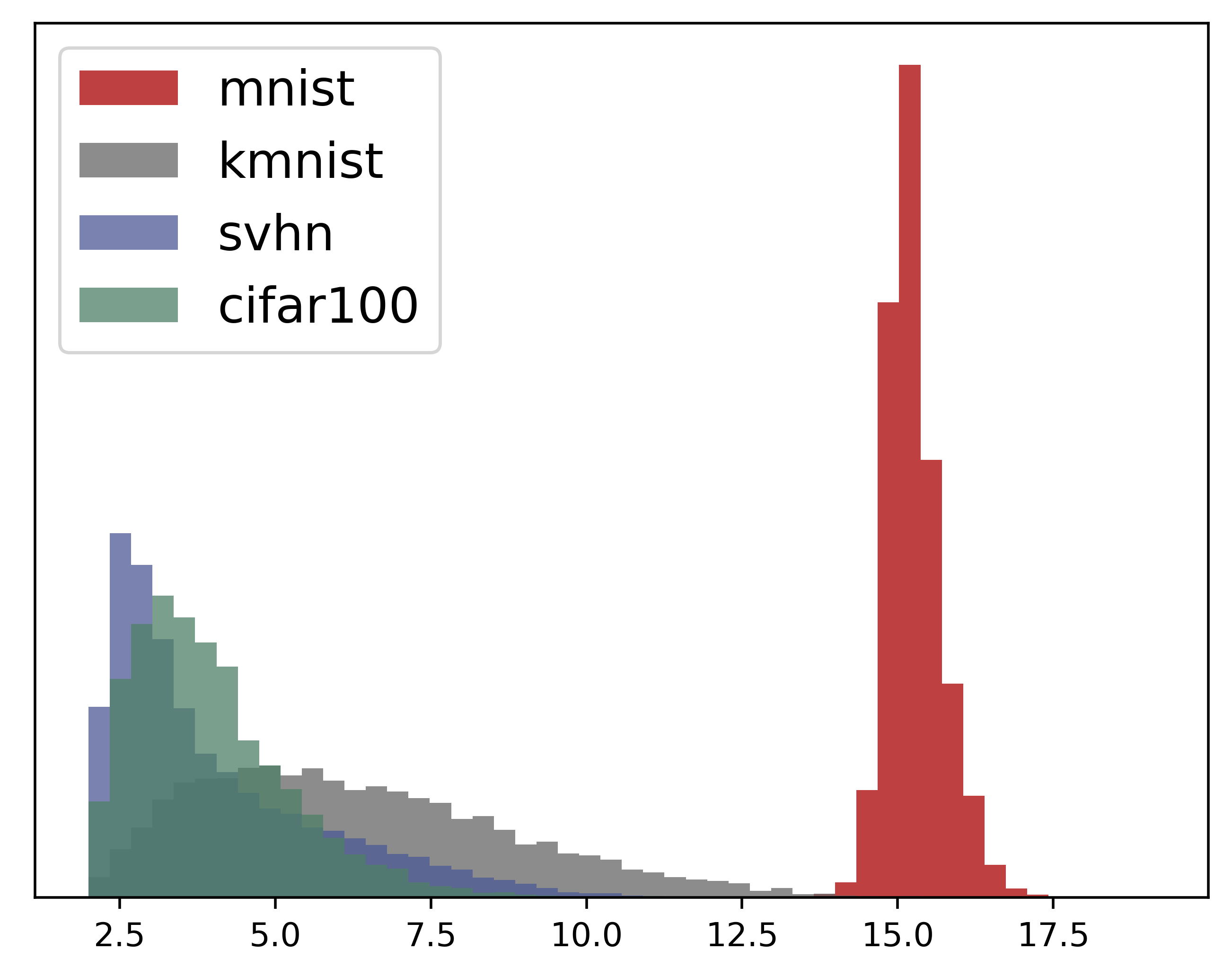}
				\label{fig:ARPL_01}
			\end{minipage}%
		}%
		\subfigure[ARPL+CS]{
			\begin{minipage}[t]{0.24\linewidth}
				\centering
				\includegraphics[width=\linewidth]{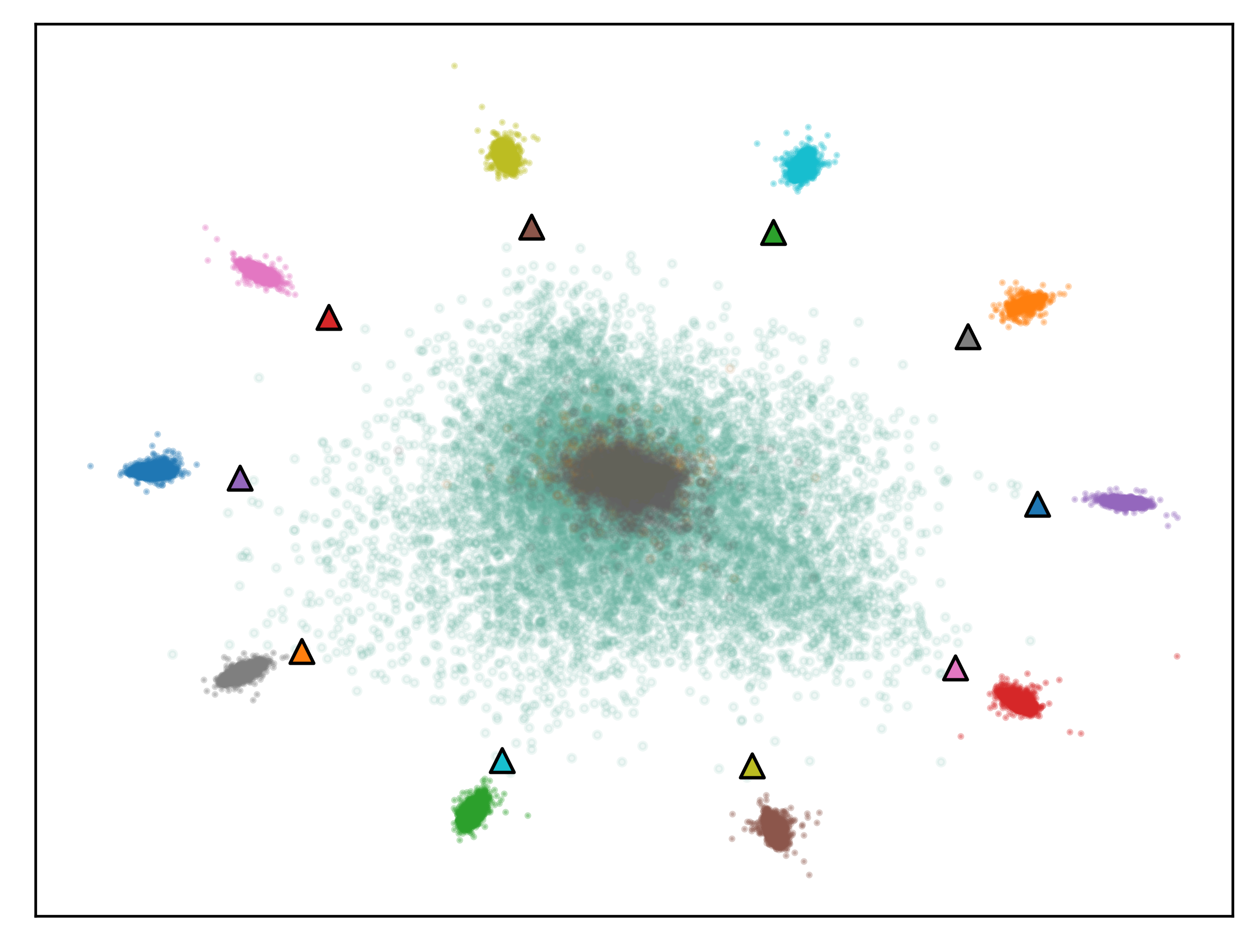}
				\includegraphics[width=\linewidth]{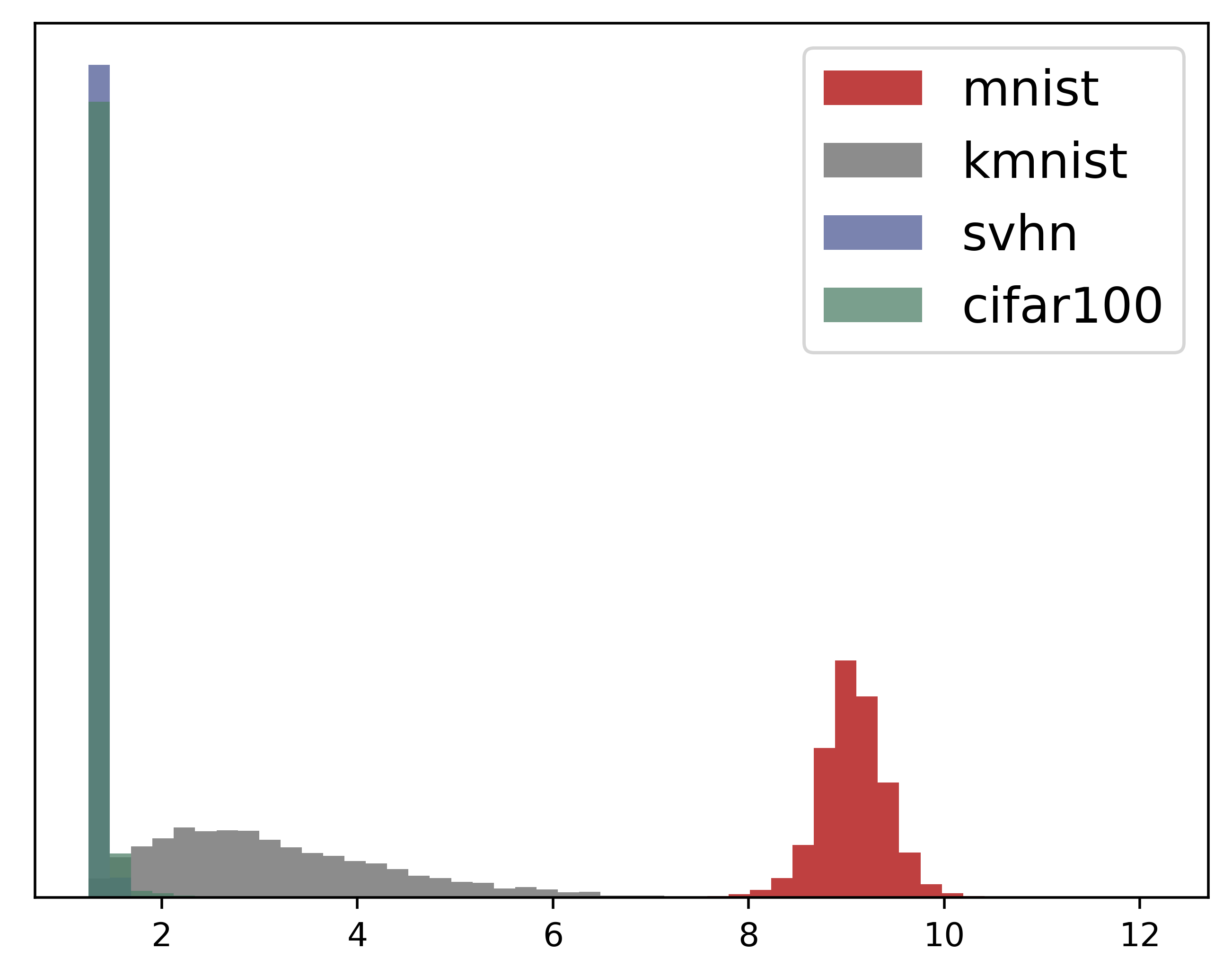}
				\label{fig:ARPL_OSS}
			\end{minipage}%
		}
		\caption{(1) The first row is visualization in the learned feature space of MNIST as known and KMNIST, SVHN, CIFAR100 as unknown. Color shapes in the middle are data of unknown, and circles in color are data of known samples. Different colors represent different classes. Colored triangles represent the reciprocal points learned of different known categories. (2) The second row is the maximum distance distribution between features and reciprocal points.}
		\label{fig:embedding}
	\end{figure*}
	
	\subsection{Ablation Study}
	\label{EXP:AS}

	\subsubsection{ARPL vs. Softmax.}
	The classification loss term (the first part of Eq. \eqref{Eqn: final-loss}) defines the classification principle of reciprocal points.
	Similar to softmax, the learned representation is still linearly separable only with this classification loss.
	As shown in Fig.~\ref{fig:intro_softmax} and \ref{fig:ARPL_0}, reciprocal points without $\mathcal{L}_o$ are learned to the origin, and there is a significant overlap between the known and unknown classes in the low response part of entire feature space. 
	Additionally, observing the OSCR curve in Fig.~\ref{fig:auc_oscr}, the CCR is lower when the FPR is lower for softmax.
	This is the reason why the neural network learned with softmax detects unknown samples as known classes with high confidence.
	In contrast, ARPL with $\mathcal{L}_o$ to constrain the open space achieves a better distribution (as shown in Fig.~\ref{fig:ARPL_01}), where the whole feature space is contained in a limited range (1-13.5 in the abscissa for unknowns and 14-17 in the abscissa for knowns in Fig.~\ref{fig:ARPL_01}) to prevent high confidence for unknown classes.
	APRL with $\mathcal{L}_o$ can guarantee high accuracy even at the low FPR in Fig.~\ref{fig:auc_oscr}, because of effective restrictions on the open space by $\mathcal{L}_o$ and pushing the known classes far away from the global open space.
	
	
	\subsubsection{ARPL vs. GCPL.} 
	As shown in Fig.~\ref{fig:intro_GCPL}, GCPL used the prototypes to reduce the intraclass variance.
	However, without considering the unknown, GCPL extends unknown classes to the whole feature space, resulting in a significant overlap with known classes.
	In the initial stage of neural network training, the prototypes of GCPL are easily distributed in the unknown feature space.
	This also leads to some known categories are distributed in the lower response part of the feature space, which increases the open space risk.
	As shown in Fig.~\ref{fig:auc_oscr}, GCPL achieves worse AUROC and OSCR performance than softmax.
	In contrast, ARPL is not affected by the initialization because each known class is far away from its corresponding unknown representation, a reciprocal point.
	Under the interaction of the classification loss and $\mathcal{L}_o$, different known categories spread to the periphery of the space, while unknown categories are restricted to the interior.
	A clear gap is maintained between the two types of samples (known vs. unknown), as shown in Fig.~\ref{fig:ARPL_01}.
	ARPL improves the robustness of neural networks by preventing the misjudgment of the unknown class through the bounded restriction, thereby enhancing and stabilizing the classification of known categories. 
	
	
	\begin{figure}[!tb]
		\centering
		\includegraphics[width=0.65\linewidth]{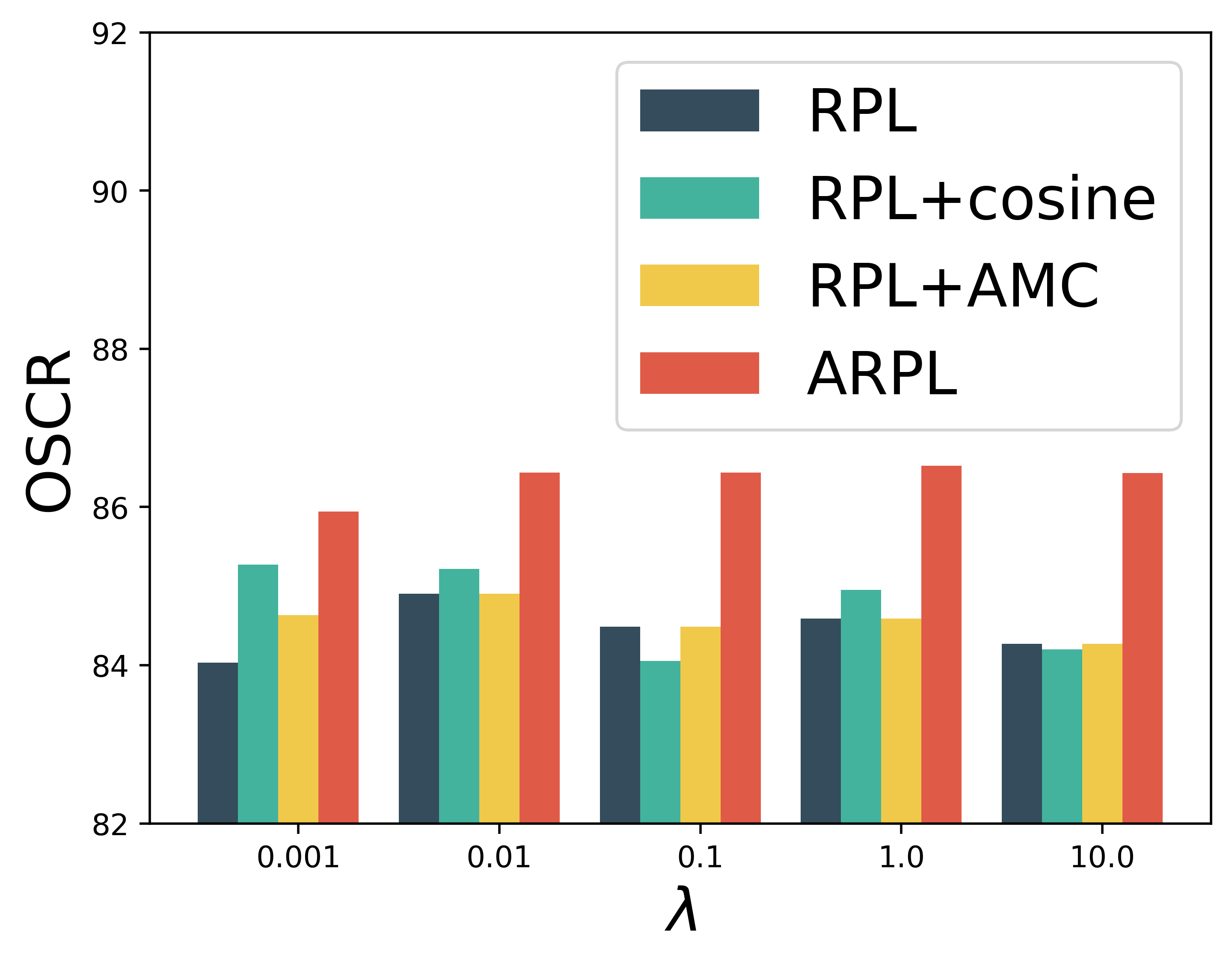}
		\caption{Ablation experiments on $\lambda$ with CIFAR10 as known data and CIFAR100 as unknown data.}
		\label{fig:lambda}
	\end{figure}
	
	\begin{figure*}[!tb]
		\centering
		\subfigure[known:MNIST, unknown: KMNIST]{
			\begin{minipage}[t]{0.31\linewidth}
				\centering
				\includegraphics[width=\linewidth]{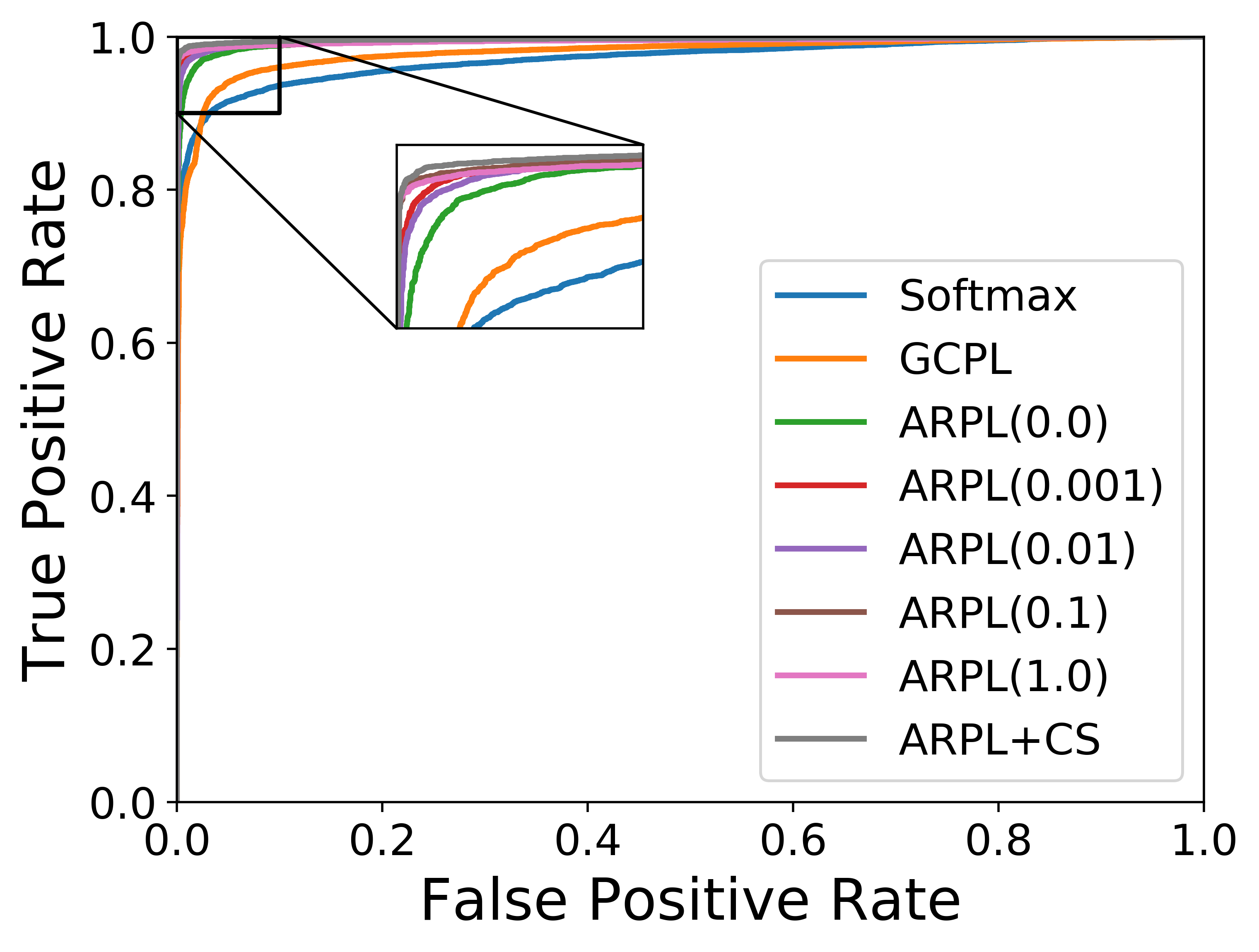}
				\includegraphics[width=\linewidth]{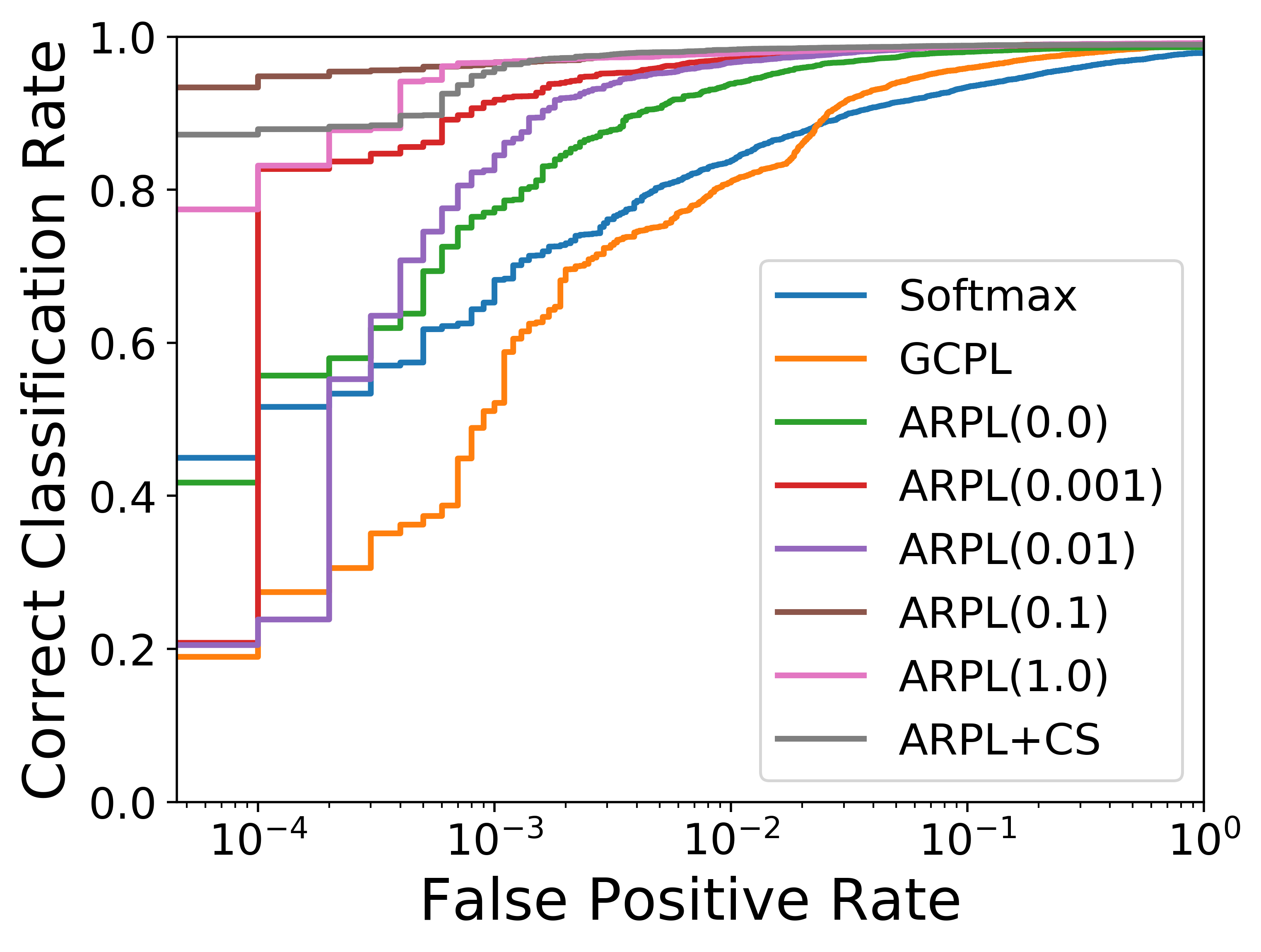}
				\label{fig:mnist_svhn_auc}
			\end{minipage}%
		}%
		\subfigure[known:MNIST, unknown: SVHN]{
			\begin{minipage}[t]{0.31\linewidth}
				\centering
				\includegraphics[width=\linewidth]{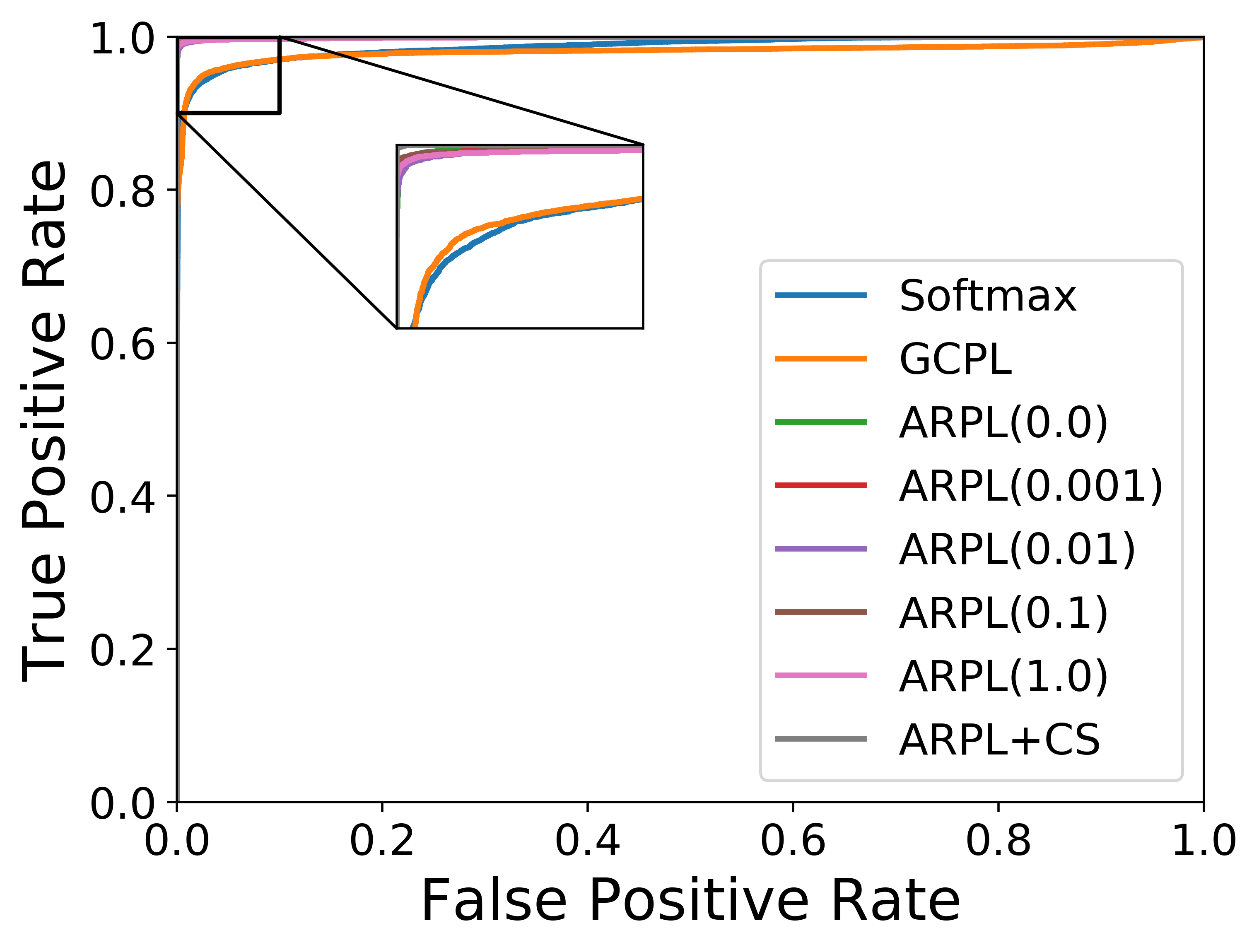}
				\includegraphics[width=\linewidth]{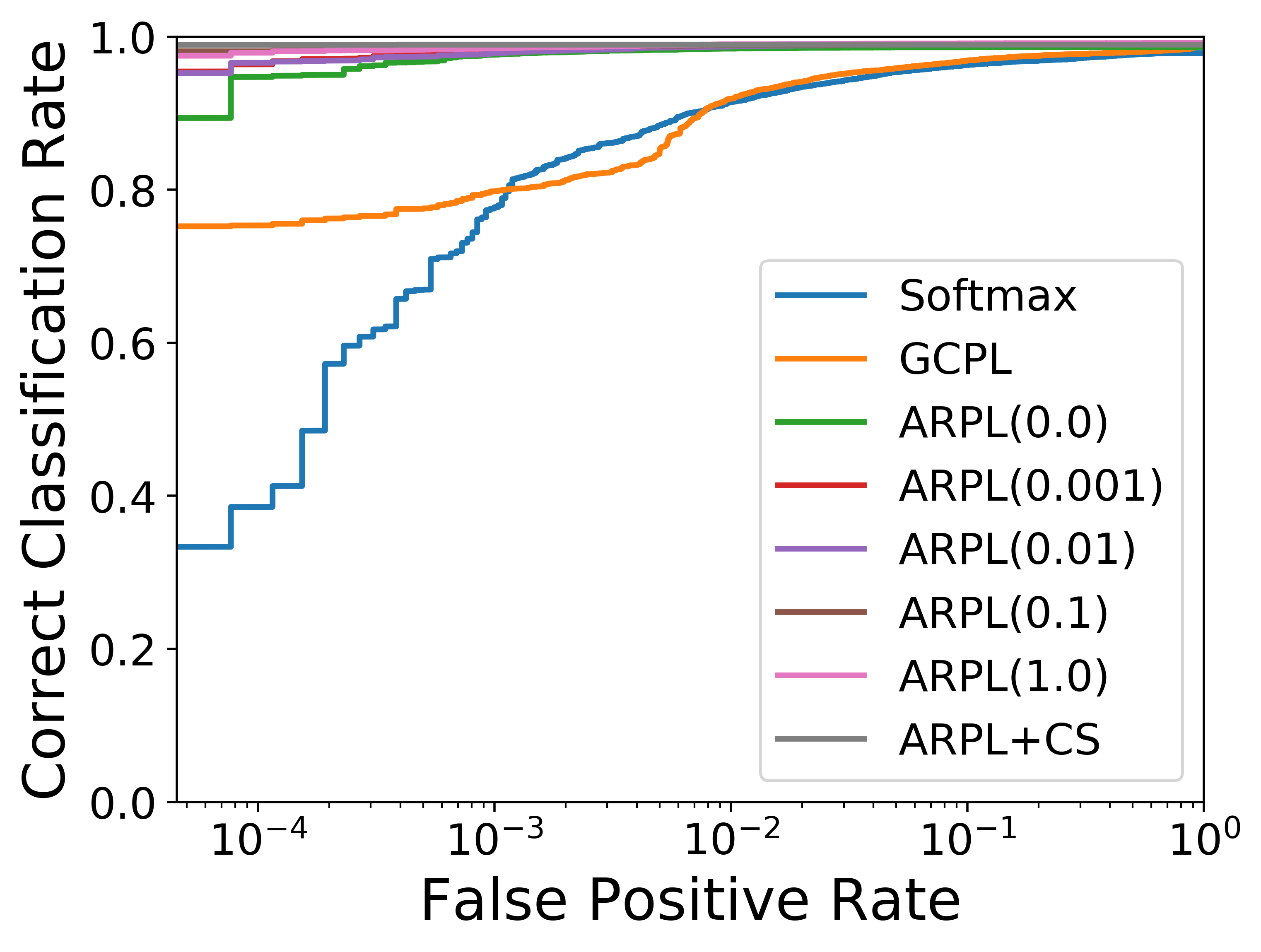}
				\label{fig:svhn_auc}
			\end{minipage}%
		}%
		\subfigure[known:MNIST, unknown: CIFAR100]{
			\begin{minipage}[t]{0.31\linewidth}
				\centering
				\includegraphics[width=\linewidth]{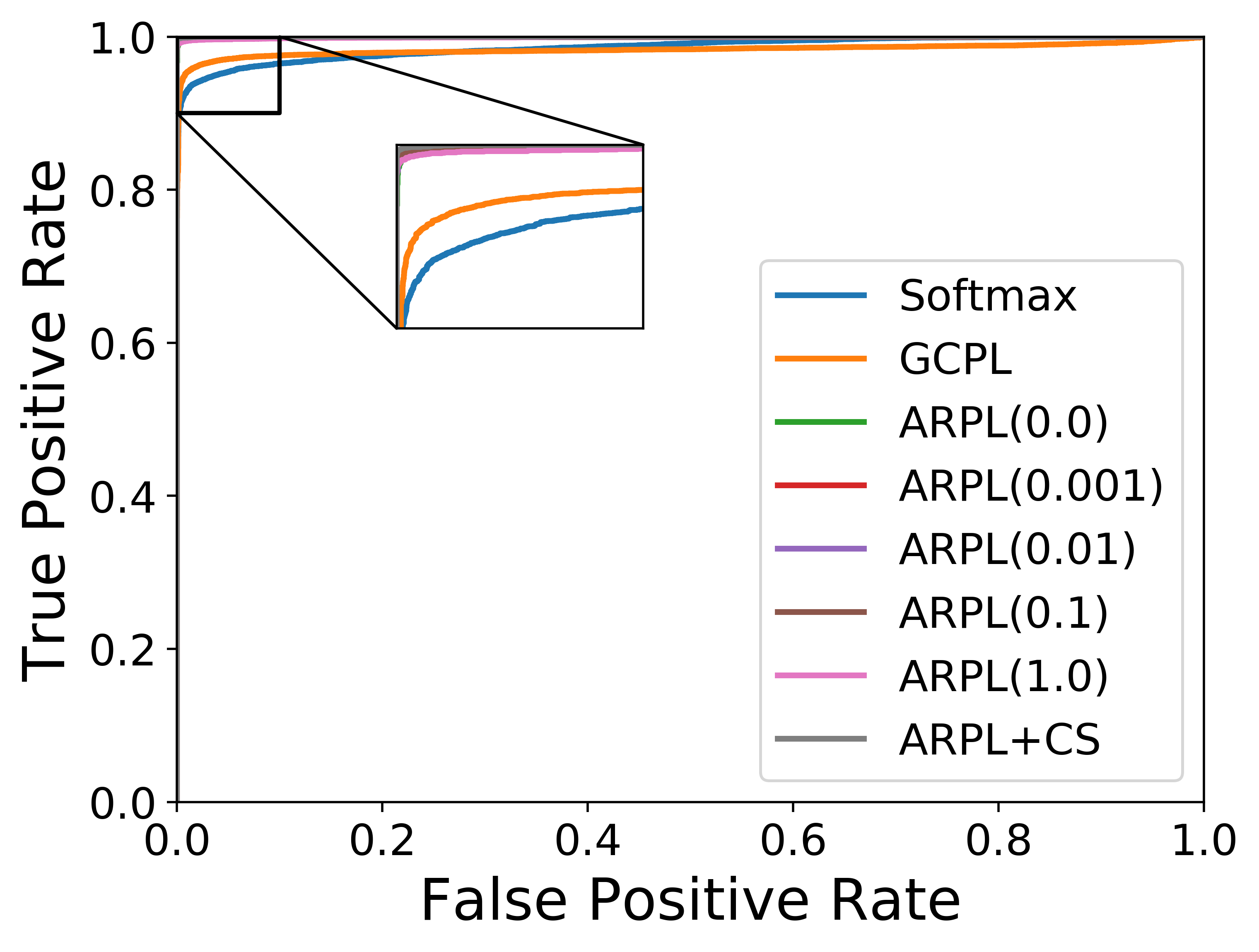}
				\includegraphics[width=\linewidth]{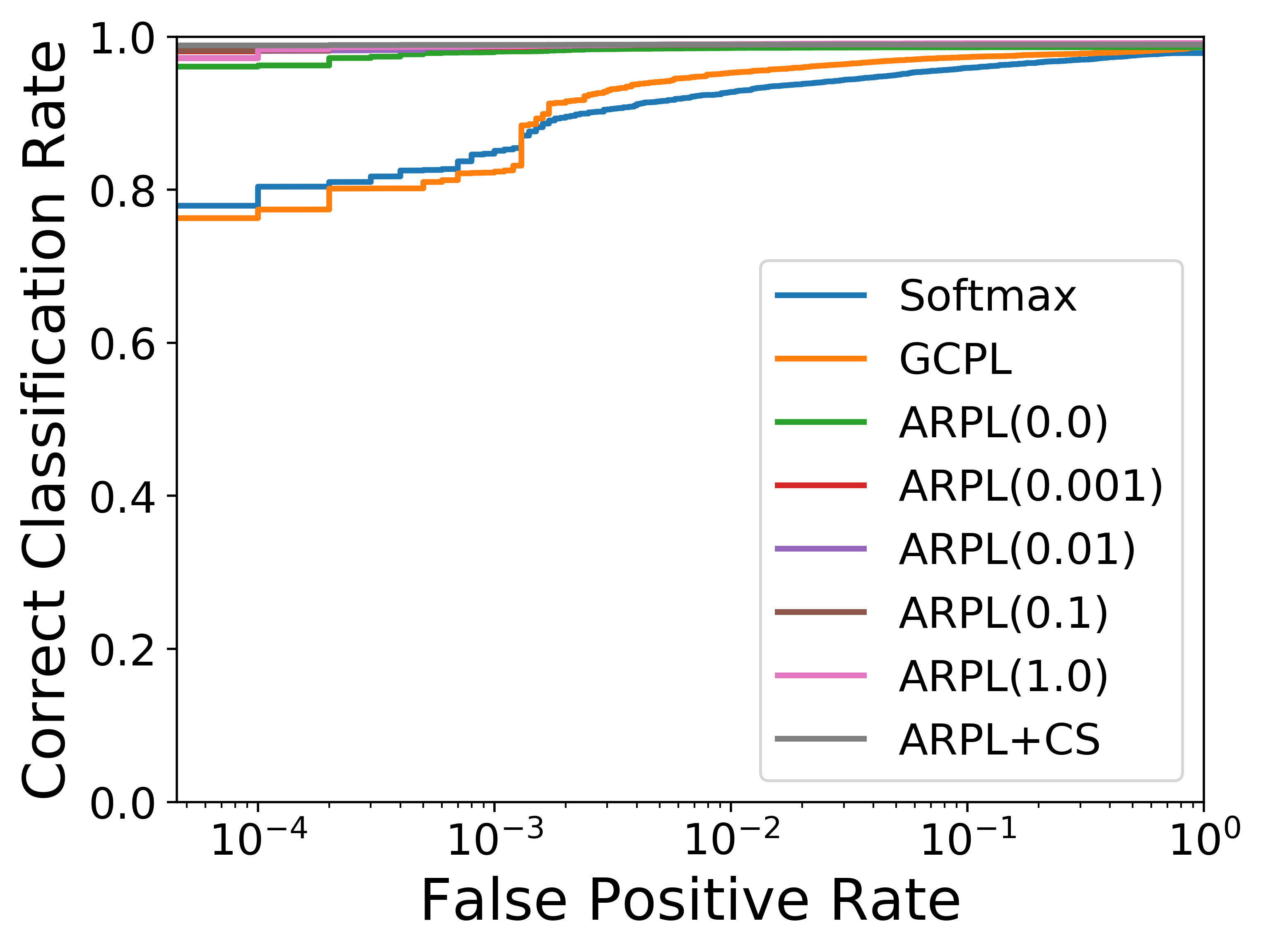}
				\label{fig:cifar100_auc}
			\end{minipage}%
		}%
		\caption{(1) The first row is the Area Under the Receiver Operating Characteristic (AUROC) applied to the data from MNIST as known and KMNIST, SVHN, CIFAR100 as unknown.  (2) The second row is Open Set Classification Rate curves provided for the same algorithms. Compared with AUROC, more significant differences could be observed through Open Set Classification Rate curves.}
		\label{fig:auc_oscr}
	\end{figure*}
	
	\subsubsection{ARPL vs. RPL.} 
	As shown in Table~\ref{tab:osdi}, Table~\ref{tab:osr} and Table~\ref{tab:ood}, ARPL has great improvement over RPL \cite{chen_2020_ECCV}. 
	Compared with RPL, ARPL improves the similarity estimation and reduces the open space risk through an elastic bounded space. 
	First, for the distance between the feature and reciprocal points in Eq. \eqref{Eqn:distance}, the angle metric is added.
	Each known class is opposite to its reciprocal points in terms of spatial position and angle direction. 
	Comparing Fig.~\ref{fig:RPL} and Fig.~\ref{fig:ARPL_0}, there are is a larger space between the features of each class in ARPL($\lambda$=0), and each class is more compact. 
	Adding the angle metric effectively reduces the intraclass distance to achieve better performance with different $\lambda$ as shown in Fig.~\ref{fig:lambda}.
	Second, we no longer limit the distance between the known classes and the corresponding reciprocal points to the same margin, and we use adaptive regularization in Eq. \eqref{Eqn:regularization}. 
	The stronger limitation in \cite{chen_2020_ECCV} will reduce the network‘s discriminability for all known classes.
	The performance of classification and unknown detection will be reduced if this restriction is too large.
	It also is sensitive to the setting of hyperparameter $\lambda$ in \cite{chen_2020_ECCV}. 
	Compared with ARPL and RPL+AMC in Fig.~\ref{fig:lambda}, the performance of RPL and RPL+cosine are more affected by $\lambda$ and their performances are more unstable.
	For the AMC, the learnable value $R$ is used as the anchor. 
	By constantly adjusting the reciprocal points and the deep feature, all known classes are promoted to less than $R$, so as to focus on samples that are difficult to distinguish among unknowns adaptively. 
	Through this adversarial margin constraint in Eq. \eqref{Eqn:regularization}, the neural network can no longer focus on the samples that meet conditions in Eq. \eqref{Eqn:reg}, and pay more attention to optimizing the bounded samples.

	\subsubsection{ARPL vs. ARPL + Confusing Samples.}
	From the experiment on OOD detection,  CS improves ARPL more for detecting far OOD. 
	Additionally, as shown in Fig.~\ref{fig:ARPL_01} and Fig.~\ref{fig:ARPL_OSS}, confusing samples make the difference between MNIST, SVHN and CIFAR100 even greater. 
	Note that SVHN and MNIST have the same class, the numbers 0-9, but they are also accurately detected as unknown classes. 
	The main reason for this difference is the difference of the image domain, color images \emph{vs.} black and white images. 
	This also demonstrates that the proposed method has the ability to reject data from different domains.
	For KMNIST, ARPL+CS does not seem to bring much improvement in feature visualization in Fig.~\ref{fig:ARPL_OSS}.
	The samples from KMNIST are difficult for the classifier to distinguish, and have more similarities in the shape and structure with those of MNIST.
	However, ARPL+CS still improves the detection ability for these confusing samples, as shown in Table~\ref{tab:ood} and Fig.~\ref{fig:auc_oscr}.
	In general, confusing samples effectively improve the ability of ARPL to detect various unknown categories while ensuring the accuracy of known class classification.
	
	\begin{figure}[htb]
		\centering
		\begin{minipage}[t]{0.48\linewidth}
			\centering
			\includegraphics[width=\linewidth]{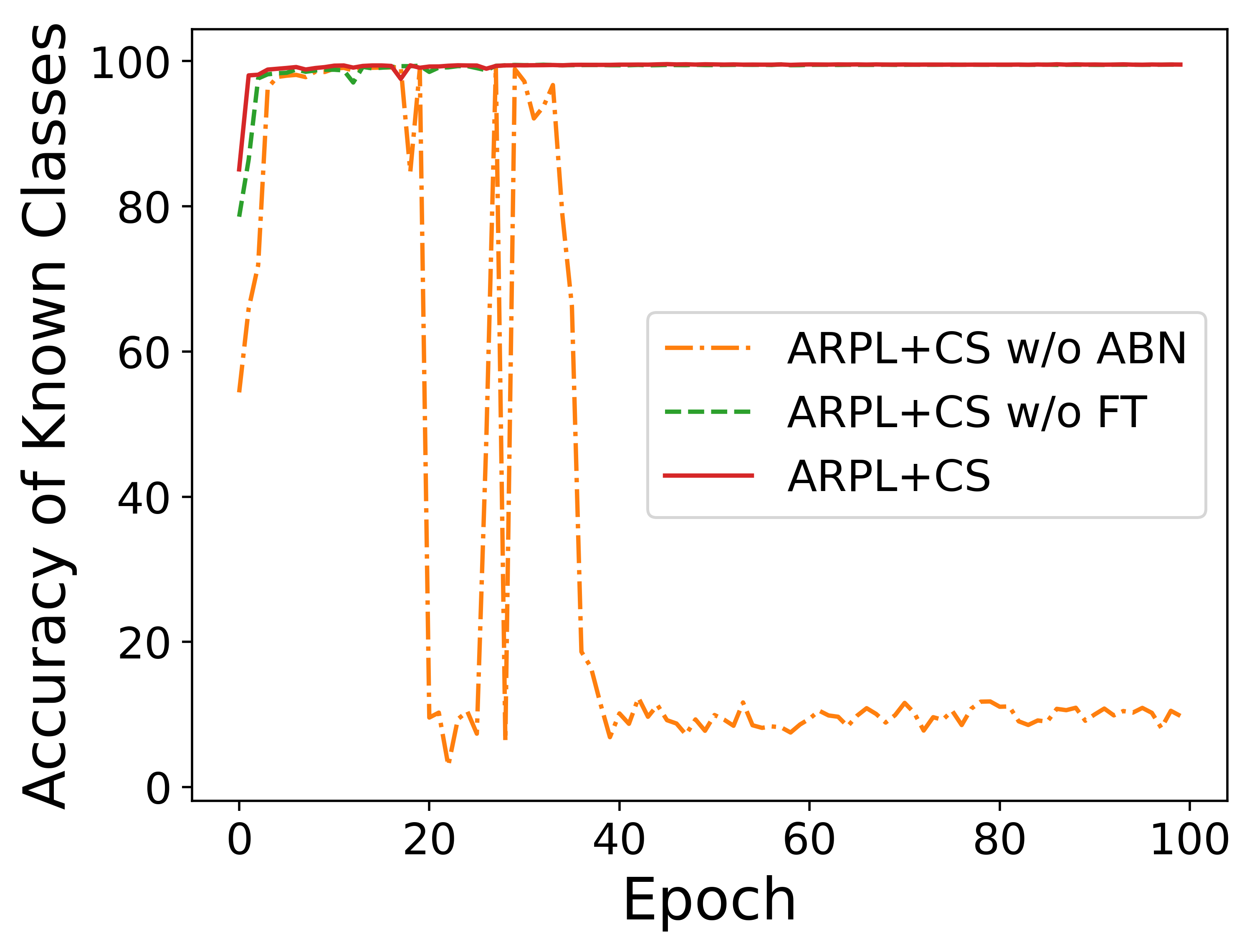}
			\label{fig:ABNFT_ACC}
		\end{minipage}%
		\begin{minipage}[t]{0.48\linewidth}
			\centering
			\includegraphics[width=\linewidth]{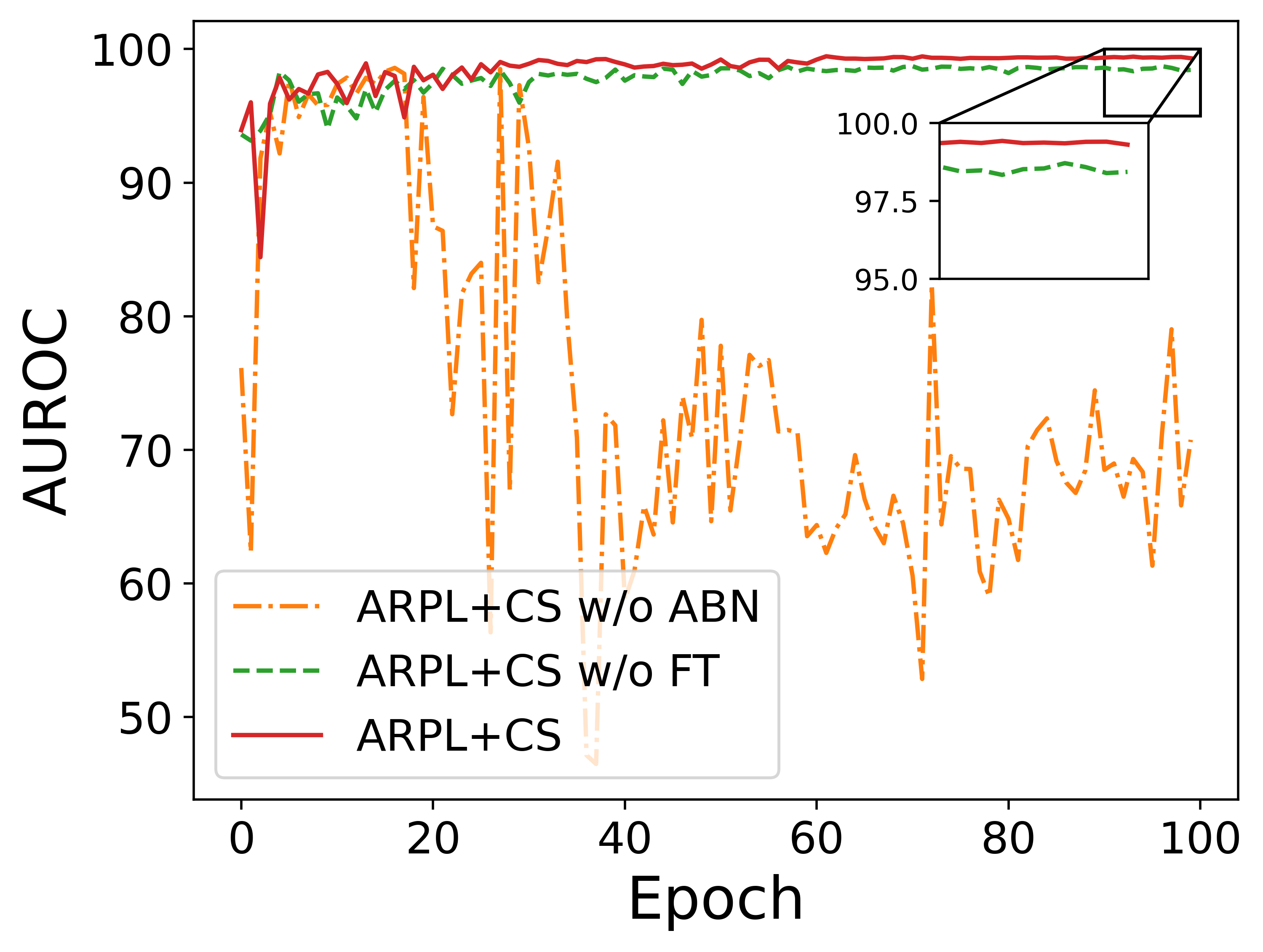}
			\label{fig:ABNFT_AUC}
		\end{minipage}%
		\vspace{-0.5cm}
		\caption{The influence of ABN and FT on the accuracy of known and AUROC for detecting unknown in the training process, where MNIST is the known dataset and KMNIST is the unknown dataset. ABN can ensure that confining samples do not made a negative effect on the classifier, and FT further enhances the ability to distinguish unknown classes.}
		\label{fig:abn}
	\end{figure}
	
	\subsubsection{ABN \& FT}
	The training process would not be stable if the two opposite constraints in Sec.~\ref{sec:gan} were added directly. Hence, we propose ABN and FT for reliability enhancement. 
	As shown in Fig.~\ref{fig:abn}, the training of ARPL without ABN finally collapses on MNIST. 
	The known and unknown classes use their own BN independently, so the training process of the known classes will not be affected when the distributions of unknown samples and known samples are very different. 
	With stability training procedure of ABN, FT further enhances the ability to distinguish known and unknown classes. With ABN and FT, ARPL+CS can improve the ability of neural networks to distinguish known classes and the discriminability of judging various unknown classes, simultaneously.

	\subsubsection{Analysis of the Margin.}
	Different datasets need different sizes for deep feature space to ensure that known and unknown can be classified correctly.
	Fig.~\ref{fig:known} proves that the margin increases with the number of known classes with fixed $\lambda$.
	Moreover, the distributions of features learned under different numbers of known classes still are discriminative for known classes and unknown classes as shown in Fig.~\ref{fig:vis}.
	As shown in Fig.~\ref{fig:multi-classes}, ARPL can also retain better performance for detecting unknown classes compared with softmax and GCPL with different numbers of known classes. 
	Softmax and GCPL only enhance the separability of features, and cannot distinguish known and unknown classes sufficiently well. 
	ARPL can make the known and unknown classes distinguishable by pushing the known samples far away from the potential unknown space.
	This phenomenon demonstrates the rationality of the spatial distribution learned for multiple classes.
	ARPL can effectively control the interaction among different known classes, by learning a more appropriate embedding space size.
	As a result, the previous conclusion about ARPL still holds for different numbers of known classes.
	
	Furthermore, for open set recognition, the fixed known classes are provided in the training, while the unknown classes are infinite. As mentioned above, reciprocal points are modeled for most unknown distribution which is  different from the provided known classes. However, there are some unknown classes that are very similar to known classes, and even it is hard to distinguish them through the human eye. In this case, the neural network will be easily affected by the limited known priors, resulting in generating a high feature response. This is also the reason that the upper left part of Fig.\ref{fig:retrieval} is affected by the open set data from KMNIST. We will explore these challenges in future research.

	\begin{figure}[!tb]
		\centering
		\subfigure[]{
			\begin{minipage}[b]{0.48\linewidth}
				\centering
				\includegraphics[width=\linewidth]{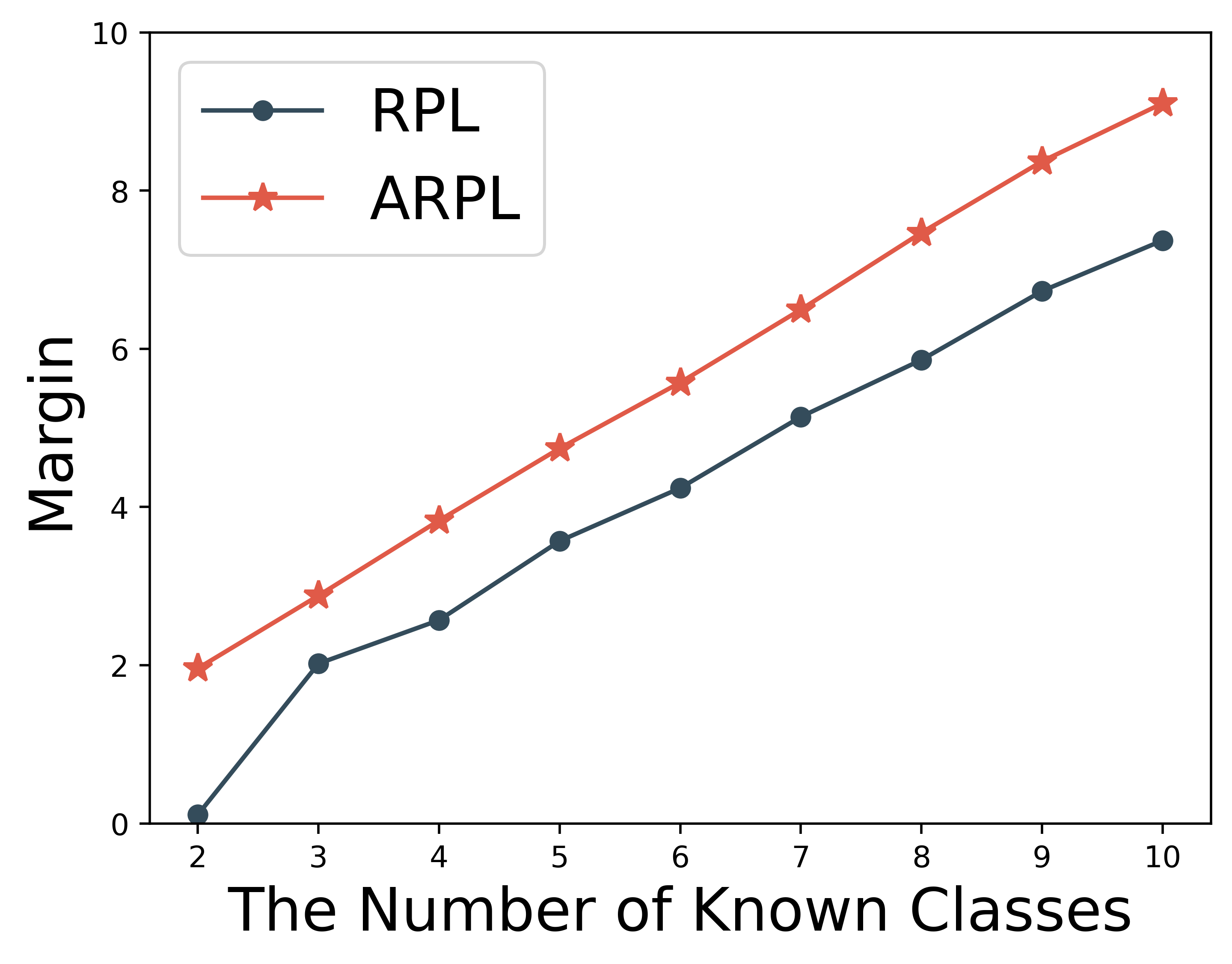}
				\label{fig:known}
			\end{minipage}%
		}
		\subfigure[]{
			\begin{minipage}[b]{0.48\linewidth}
				\centering
				\includegraphics[width=\linewidth]{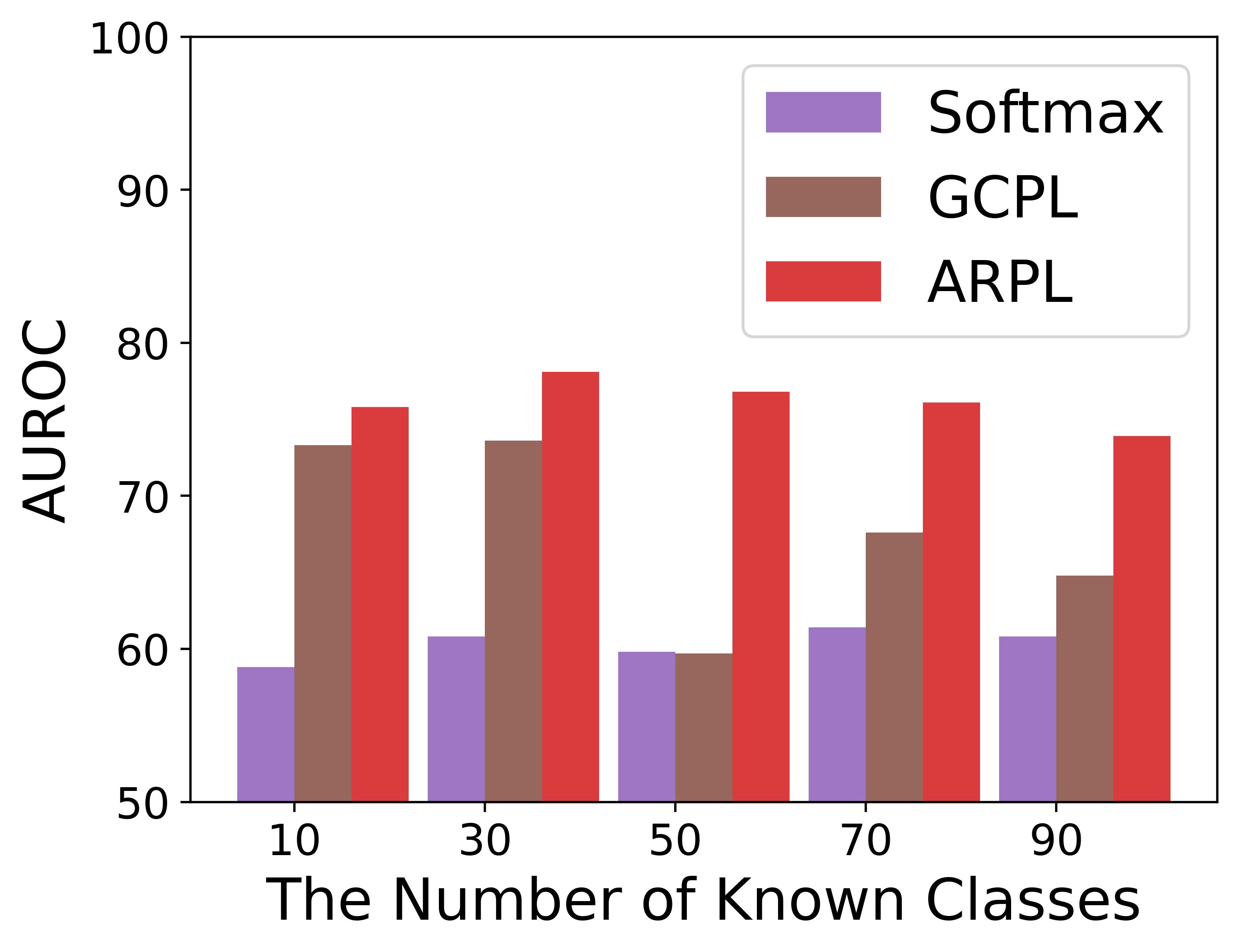}
				\label{fig:multi-classes}
			\end{minipage}%
		}
		\caption{(a): The variation trend of margin with the number of known classes. (b): The AUROC performance based on the different number of known classes on CIFAR100, where remaining classes as unknown.}
		\label{fig:classes}
	\end{figure}

	\begin{figure*}[!tb]
		\centering
		\subfigure[$K = 2$]{
			\begin{minipage}[t]{0.24\linewidth}
				\centering
				\includegraphics[width=\linewidth]{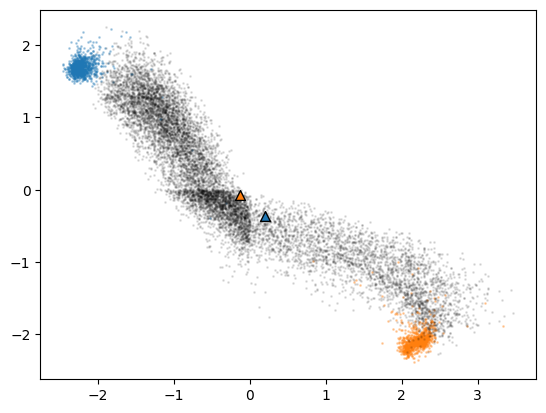}
				\label{fig:ARPL_2}
			\end{minipage}%
		}%
		\subfigure[$K = 4$]{
			\begin{minipage}[t]{0.24\linewidth}
				\centering
				\includegraphics[width=\linewidth]{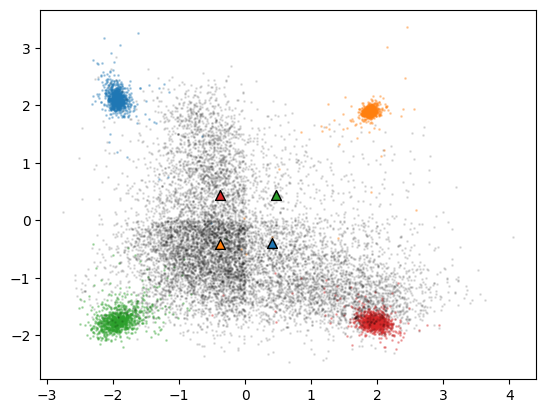}
				\label{fig:ARPL_4}
			\end{minipage}%
		}%
		\subfigure[$K = 6$]{
			\begin{minipage}[t]{0.24\linewidth}
				\centering
				\includegraphics[width=\linewidth]{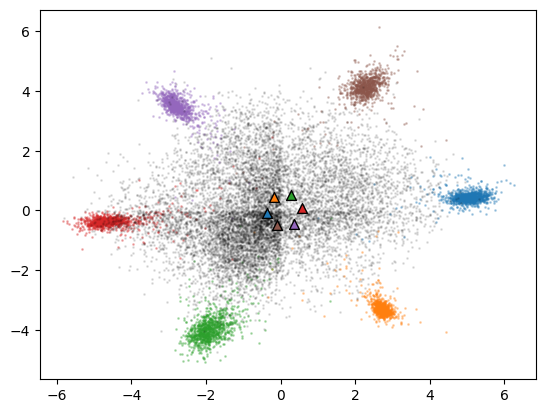}
				\label{fig:ARPL_6}
			\end{minipage}%
		}%
		\subfigure[$K = 8$]{
			\begin{minipage}[t]{0.24\linewidth}
				\centering
				\includegraphics[width=\linewidth]{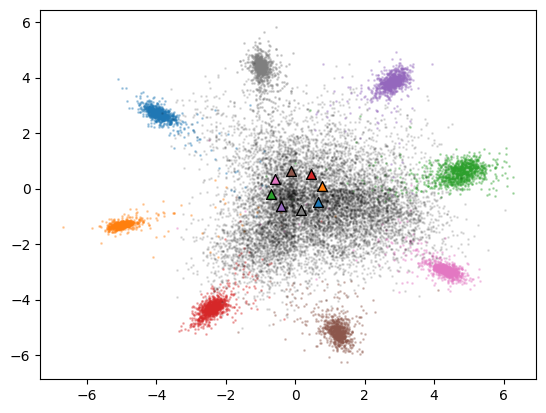}
				\label{fig:ARPL_8}
			\end{minipage}%
		}
		\caption{
			The learned representations of ARPL with different numbers of known classes. The data is from MNIST by randomly sampling $K$ known classes and $10-K$ unknown classes.  Colored triangles represent the learned reciprocal points of different known classes.}
		\label{fig:vis}
	\end{figure*}
	
	\subsection{Further Analysis}
	\label{EXP:FA}

	\subsubsection{Analysis of Closed Set Recognition.}
	We adopt ResNet with 34 layers \cite{he2016deep} for closed set recognition on CIFAR10, CIFAR100, and \textbf{Air}craft \textbf{300} (Air-300) \cite{chen_2020_ECCV}.
	Air-300 contains 320,000 annotated color images from 300 different classes in total.
	Each category contains 100 images least, and a maximum of 10,000 images, which leads to a long-tailed distribution.
	All classes are divided into two parts with 180 known classes for training and 120 novel unknown classes for testing respectively.
	In contrast to the existing benchmark datasets, the tailored Air-300 dataset maintains a long-tailed distribution to simulate the real visual world.
	Here, we focus on the closed set accuracy of the model in 180 known categories.
	The images of Air-300 are center-cropped and resized to 64x64 in this experiment.
	
	As shown in Table~\ref{tab:close}, ARPL achieves comparable performance with traditional softmax and the prototype learning method GCPL. 
	The interclass distance between the known classes is increased by using reciprocal points to push the known classes away from the global open space, so that the neural network can learn more discriminative features for the known classes through ARPL.
	This demonstrates the effectiveness of ARPL for conventional closed set recognition tasks. 
	Moreover, by incorporating CS, the closed set accuracy for ARPL+CS is not decreased and it can even perform better ARPL on CIFAR100 and Air-300.
	ARPL+CS is not affected by confusing samples, which largely depends on the proposed ABN and FT.
	Without ABN or FT, the closed set accuracy can decrease because of the deviation from confusing samples.
	These results demonstrate that ARPL and ARPL+CS can improve the ability of neural networks to distinguish known classes and the discriminability of judging various unknown classes.
	
	\begin{table}[!tb]
		\caption{Test accuracy of different methods on CIFAR10, CIFAR100 and Air-300. The best results are indicated in bold.}
		\centering
		\label{tab:close}
		\begin{tabular}{ccccc}
			\specialrule{.16em}{0pt} {.65ex}
			Method & CIFAR10 & CIFAR100 & Air-300 \\
			\midrule
			Softmax & 93.1 & 70.8 & 92.4 \\
			GCPL & 93.3 & 70.3 & 92.3 \\
			RPL & 93.8 & 71.8 & 92.9 \\
			ARPL & \textbf{94.1} & 72.1 & 94.5 \\
			ARPL+CS(w/o ABN) & 94.0 & 71.8 & 93.2 \\
			ARPL+CS(w/o FT) & 93.1 & 71.5 & 93.0 \\
			ARPL+CS & 94.0 & \textbf{72.8} & \textbf{94.7} \\
			\specialrule{.16em}{.4ex}{0pt}
		\end{tabular}
	\end{table}
	
	\subsubsection{Semantic Shift versus Non-semantic Shift}
	More complex open set scenarios is explored through a large-scale data set DomainNet \cite{peng2019moment}. 
	DomainNet has high-resolution images in 345 classes from six different domains. 
	There are three domains in the dataset with class labels available when the experiments are conducted. 
	These are real, clipart, and quickdraw, and they result in different types of distribution shifts.
	To create subsets with semantic shifts, all classes are separated into two splits.
	Split A has class indices from 0 to 172, while split B has indices from 173 to 334.
	Our experiments use real-A for in-distribution and the other subsets for out-of-distribution.
	With the definition given in \cite{hsu2020generalized}, real-B has a semantic shift from real-A, while clipart-A has a nonsemantic shift.
	Clipart-B therefore has both types of distribution shift.
	We train the ResNet with 34 layers \cite{he2016deep} for 100 epochs with batch size 128 and an SGD optimizer with momentum 0.9.
	The learning rate starts at 0.01 and decreases by a factor of 0.1 in the training process every 30 epochs.
	The images are center-cropped and resized to 80x80 in this experiment.

	The results in Table~\ref{tab:domain} reveal some trends.
	The first is that the OOD datasets with both types of distribution shifts are easier to detect, followed by nonsemantic shifts.
	The second observation is that ARPL can effectively detect all distribution shifts, in contrast to softmax.
	In particular, for OSCR, ARPL achieves a good performance improvement.
	Finally, confusing samples play an important role in different domains and can improve the detection performance of ARPL. 
	The near domain can obtain a larger improvement from confusing samples. 
	
	\begin{table}[!tb]
		\caption{Performance of three methods using DomainNet. The known is the real-A subset. The type of distribution shift presents a trend of difficulty to the OOD detection problem: Semantic shift (S) $>$ Non-Semantic Shift (NC) $>$ Semantic + Non-Semantic shift.}
		\centering
		\small
		\label{tab:domain}
		\begin{tabular}{ccccc}
			\specialrule{.16em}{0pt} {.65ex}
			OOD & \multicolumn{2}{c}{Shift} & AUROC & OSCR \\
			\midrule
			& S & NS & \multicolumn{2}{c}{Softmax / ARPL / ARPL+CS} \\
			\midrule
			real-B & \checkmark & & 72.3/74.2/\textbf{75.2} & 41.7/60.8/\textbf{61.9} \\
			clipart-A & & \checkmark & 66.4/70.9/\textbf{72.7} & 41.4/58.0/\textbf{59.4} \\ 
			clipart-B & \checkmark & \checkmark & 77.0/81.1/\textbf{82.9} & 45.3/65.3/\textbf{66.6} \\
			quickdraw-A & & \checkmark & 77.9/86.1/\textbf{86.7} & 44.6/68.5/\textbf{69.0}\\
			quickdraw-B & \checkmark & \checkmark & 79.5/87.4/\textbf{87.5} & 45.5/69.4/\textbf{69.5} \\
			\specialrule{.16em}{.4ex}{0pt}
		\end{tabular}
	\end{table}

	\subsubsection{Experiments on ImageNet.}
	
	To better compare our method with traditional softmax, we conduct experiments on the larger and more difficult ImageNet-1K dataset \cite{ILSVRC15}. ImageNet-1K includes 1000 classes with more than 1,200,000 training images and 50K validation images. Moreover, ImageNet-O \cite{hendrycks2019natural} is adopted as the out-of-distribution dataset for ImageNet-1K. ImageNet-O includes 2K examples from ImageNet-22K \cite{ILSVRC15} excluding ImageNet-1K. ResNet 18 \cite{he2016deep} is trained on ImageNet-1K and tested on both ImageNet-1K and ImageNet-O.
	
	As shown in Table~\ref{tab:imagenet}, ARPL performs better than traditional softmax, GCPL and RPL even on the large and difficult datasets, in terms of both close-set accuracy (ACC) and unknown detection (AUROC). In particular, for unknown detection, ARPL achieves an approximately 12\% improvement over softmax. Due to the constraints on the global open space imposed by the reciprocal points, our method can ensure a better separation of known and unknown classes while ensuring the accurate recognition of known classes. These results show the excellent scalability of ARPL in larger scale datasets.
	
	\begin{table}
		\caption{Open set recognition performance of different methods on the larger and more difficult datasets, where ImageNet-1K as the known dataset and ImageNet-O as the unknown dataset.}
		\centering
		\label{tab:imagenet}
		\begin{tabular}{cccc}
			\specialrule{.16em}{0pt} {.65ex}
			Method & ACC & AUROC & OSCR \\
			\midrule
			Softmax & 69.6 & 48.2 & 42.4 \\
			GCPL & 63.0 & 56.0 & 43.0 \\
			RPL & 69.8 & 59.4 & 48.6 \\
			ARPL & \textbf{70.2} & \textbf{60.0} & \textbf{48.9} \\
			\specialrule{.16em}{.4ex}{0pt}
		\end{tabular}
	\end{table}
	
	\subsubsection{Extension for Class-Incremental Learning.}
	To validate the potential  of reciprocal points in incremental learning, we train ResNet with 34 layers to classify CIFAR100.
	We assume that a classifier is pretrained on a certain number of base classes and new classes with corresponding datasets are incrementally provided one by one.
	In this incremental scenario, half of the CIFAR100 classes are set as base classes and the rest are set as new classes.
	The experiment are conducted in five times where class splits are randomly generated in each time, and then the averaged results are reported.
	Since the reciprocal points of the base classes are distributed in the unknown deep feature space which may contain the novel classes, we try to use the reciprocal points to represent the novel classes.
	For comparison, we consider different sets of bases to represent base classes and novel classes, where the bases are learned by other methods. Specifically,
	for the softmax classifier, the weight of the last linear layer is used as a set of bases, so that the logit is the coefficient based on these bases.
	Similarly, the prototypes and reciprocal points of base classes can be different sets of bases to represent different classes and their logit scores are used for classification.

	\begin{figure}
		\centering
		\includegraphics[width=0.8\linewidth]{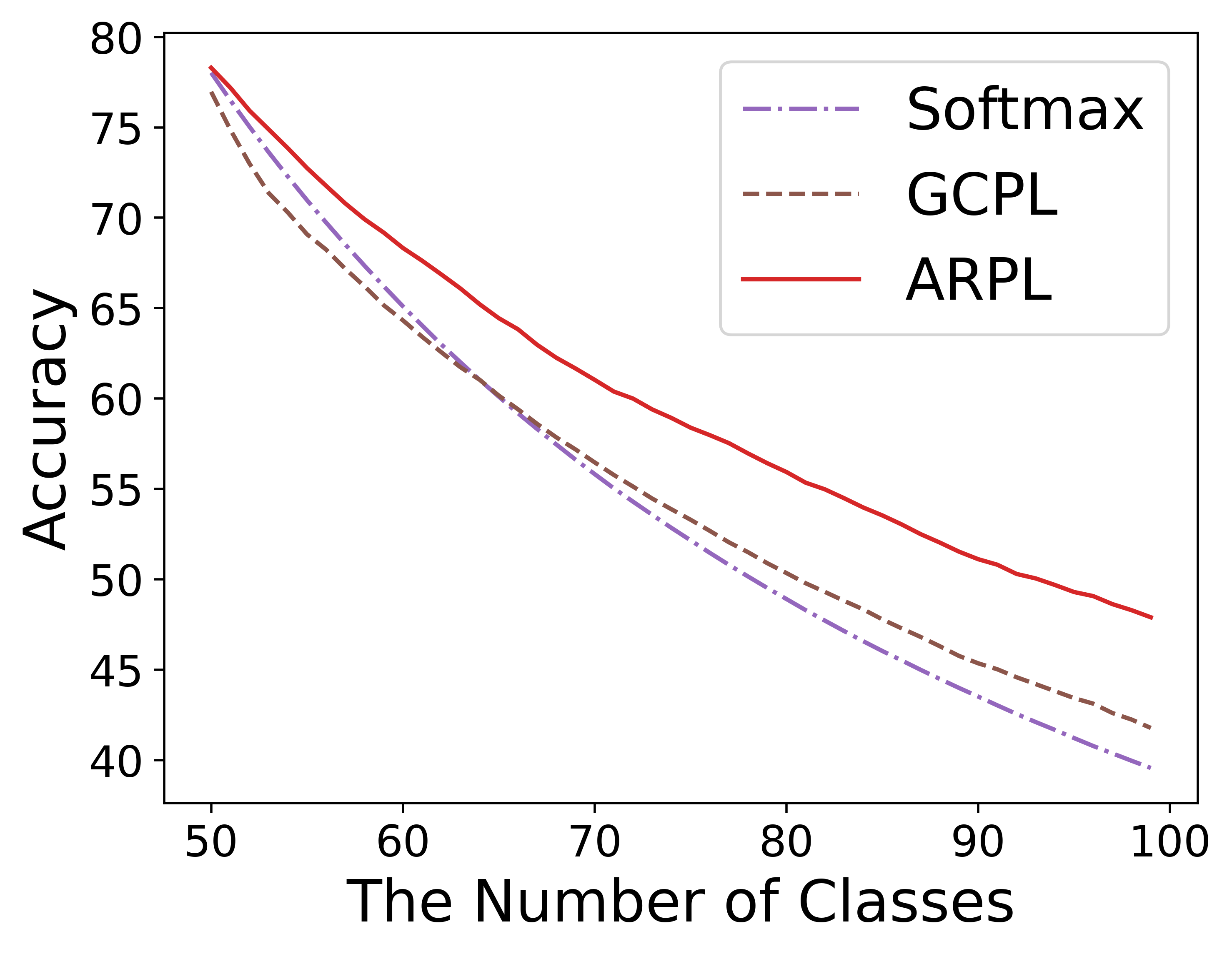}
		\caption{Experimental results of class-incremental learning on CIFAR100. We report the overall accuracy after the last new class is added.}
		\label{fig:incremental}
		\vspace{-0.5cm}
	\end{figure}
	
	The experiments are evaluated by  overall accuracy and the results are shown in 
	Fig.~\ref{fig:incremental}.
	The proposed reciprocal points outperform the other methods by a significant margin, as the number of new classes increases.
	In the case of similar initial accuracy, reciprocal points can better resist catastrophic forgetting.
	The experimental results demonstrate the superiority of our reciprocal points.
	Further potential of reciprocal points in incremental learning will be explored in the future.

	\section{Conclusion}
	This paper formulates the open space risk from the perspective of multiclass integration, by introducing a novel concept, the \emph{Reciprocal Point}, to model the extraclass space corresponding to each known category. 
	We introduce a novel learning framework, Adversarial Reciprocal Point Learning, to promote a reliable open set neural network. 
	Specifically, a classification framework with the adversarial margin constraint is introduced to reduce the empirical classification risk and the open space risk. 
	The rationality of the adversarial margin constraint is theoretically guaranteed by Theorem \ref{theorem:multi}.
	Furthermore, to estimate the unknown distribution from the open space, an instantiated adversarial enhancement is designed to generate more diverse confusing training samples from the confrontation between the known data and reciprocal points.
	Our methods break the closed-world assumption in traditional neural networks and adopt open-world reciprocal points for discrimination between known and unknown samples.
	Extensive experiments conducted on multiple datasets demonstrate that our method outperforms previous state-of-the-art open set classifiers in all cases.  
	
	This paper also reveals that the recognition of unknown classes by a neural network is mostly based on known priors, so the distribution of unknown classes is more aggregated in the low-response area of deep feature space, while the known classes are distributed in the high response space. This is very similar to the observation that the neocortical areas obtain structured knowledge from the hippocampus through interleaved learning \cite{kumaran2016learning}. In the future, we will explore more details about the neural mechanism of few shot learning, and then utilize it to improve the ability of the neural network to detect and learn unknown categories.

	\section*{Acknowledgments}
	This work is partially supported by grants from the Key-Area Research and Development Program of Guangdong Province under contact No.2019B010153002, and grants from the National Natural Science Foundation of China under contract No. 61825101 and No. 62088102.
	
	\bibliographystyle{unsrt}
	\bibliography{main}

\end{document}